\documentclass[10pt,twocolumn,letterpaper]{article}

\usepackage[algorithms,pagenumbers]{wacv}

\usepackage{amsmath,amssymb}
\usepackage{booktabs}          % \toprule \midrule \bottomrule
\usepackage{multirow}
\usepackage{enumitem}
\usepackage{amsthm}
\usepackage{float}
\usepackage{algorithm}
\usepackage{algpseudocode}
\newtheorem{proposition}{Proposition}
\definecolor{wacvblue}{rgb}{0.21,0.49,0.74}
\usepackage[pagebackref,breaklinks,colorlinks,allcolors=wacvblue]{hyperref}
\usepackage[capitalise,noabbrev,capitalize]{cleveref}   % must load AFTER hyperref

\def\wacvPaperID{2805}
\def\confName{WACV}
\def\confYear{2027}

\title{Predicted Class Correction: Conformal Thresholds for Classes\\with No Calibration Data}

\author{One Octadion$^{1}$\thanks{Corresponding author.} \quad Novanto Yudistira$^{1}$ \quad
Yuta Nakashima$^{2}$\\[4pt]
{\normalsize $^{1}$Universitas Brawijaya, Malang, Indonesia \quad
$^{2}$SANKEN, The University of Osaka, Osaka, Japan}\\[2pt]
{\normalsize\texttt{octadion@student.ub.ac.id} \quad \texttt{yudistira@ub.ac.id} \quad
\texttt{n-yuta@im.sanken.osaka-u.ac.jp}}
}

\begin{document}
\maketitle

\begin{abstract}
Class-conditional conformal prediction calibrates a separate score threshold for every class
and therefore requires labelled calibration examples of every class. In long-tailed
deployments these examples are precisely what is missing, as most rare classes are consumed
by training and are left with few or no held-out examples. Existing class-conditional methods
are then either undefined for such a class or fall back to the shared marginal threshold. We
ask whether the per-class correction can instead be predicted from descriptors of the class
that require no labels at all.
We propose Predicted Class Correction (PCC), which regresses the offset between the per-class
and the marginal conformal quantile on geometric descriptors of the classifier head, shrinks
the prediction by a scalar selected on the labelled classes, and refits a single marginal
offset. The descriptors are functions of the model's weights, so a class with zero
calibration examples still receives a threshold specific to it, and only one scalar is
calibrated, so split conformal's marginal guarantee is retained on the label space it is
calibrated on (Prop.~\ref{prop:marginal}).
On ImageNet, PCC improves worst-class coverage on such classes by $0.058$ $[0.006,0.110]$ at
matched average set size, $42\%$ of an oracle that observes the test labels. Of eight
baselines run on the same splits, six change the statistic by exactly zero and the best
reaches $0.016$ with an interval containing zero, so PCC is the only one whose improvement is
distinguishable from none; two of the eight are undefined on all $300$ classes. The improvement
persists with only $100$ labelled classes out of $1000$.
A second finding concerns how such methods are evaluated. Worst-class coverage becomes
unmeasurable before it becomes unachievable: below roughly $35$ test examples per class no
improvement can be demonstrated even for an oracle given the test labels, so a null result
there is a property of the evaluation set rather than of the method. Most classes of both
long-tailed benchmarks sit below it, and restricting evaluation above it recovers a significant
improvement on Pl@ntNet-300K, where the effect had appeared absent.
\end{abstract}

\section{Introduction}
\label{sec:intro}

A prediction set replaces one fallible prediction with a short list that is likely to contain
the true label. Conformal prediction turns any trained classifier into such a procedure, with
a finite-sample coverage guarantee and no assumptions on the data distribution
\cite{vovk2005algorithmic,shafer2008tutorial,angelopoulos2023gentle}.

\begin{figure}[t]
\centering
\includegraphics[width=\linewidth]{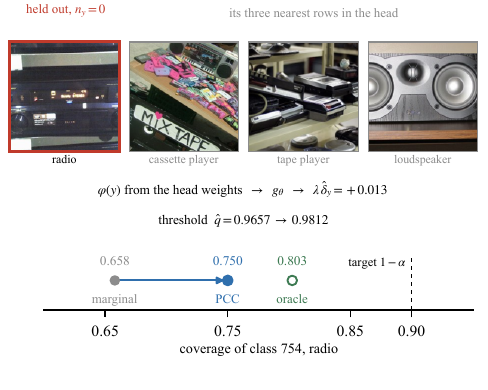}
\caption{A class with no calibration example still receives a threshold of its own. ImageNet
class $754$ is held out, so $n_y=0$: every classwise method either leaves it undefined or
hands it the shared marginal threshold. Its descriptor $\varphi(y)$ is read from the
classifier head, whose three nearest rows to it are other audio equipment, and the fitted
regression turns that geometry into an offset. The threshold moves from $\hat q = 0.9657$ to
$0.9812$ and coverage of the class from $0.658$ to $0.750$, against $0.803$ for an oracle
given the test labels. The class shown is the held-out class with the lowest coverage under
the marginal threshold, fixed by that rule before any image was looked at.}
\label{fig:teaser}
\end{figure}

The guarantee that comes for free is marginal: averaged over all inputs, the set contains the
true label with probability at least $1-\alpha$. Since marginal coverage is a prevalence-weighted
average of the class-conditional coverages, it is dominated by the frequent classes, and a
single shared threshold systematically under-covers the hard and rare ones, and coverage conditional
on every input is unattainable in finite samples without further assumptions
\cite{lei2014distribution,barber2021limits}. In biodiversity monitoring
and in diagnosis the rare classes are often the ones that matter most, so the label a user most
needs to see is the one the set most often omits. Recovering per-class control requires
calibrating a separate threshold for each class, which is what classwise conformal prediction
\cite{vovk2005algorithmic,vovk2012conditional,sadinle2019least} and its successors
\cite{ding2023class,shi2024rank,ding2026conformal} do.

Those methods all require labelled calibration examples of the class whose threshold they are
setting. In long-tailed deployments \cite{zhang2023deep} most rare classes are consumed by model training, which,
as \citet{ding2026conformal} note, leaves them with ``few or zero holdout examples to use for
calibrating uncertainty quantification methods.'' The zero case is common: the median class has three
held-out examples on Pl@ntNet-300K \cite{garcin2021plantnet} and two on iNaturalist-2018
\cite{vanhorn2018inaturalist}. What existing methods do with such a class is
summarised in \hyperref[tab:methods]{Tab.~S1} of the supplementary material. \textsc{Classwise} CP and \textsc{Rc3p} return an infinite
threshold, so the set becomes the entire label space; \textsc{Clustered} CP falls back to the
shared marginal threshold; and the kernel-weighted method of \citet{ding2026conformal} embeds
each class by its own calibration scores, so classes with none collapse to a single point and
share one threshold. None of this is unreasonable, since with no labelled example of a class
there is nothing to condition on.

\paragraph{Objective.} Something about a class is nevertheless known without any of its
labels. A trained classifier already places every class somewhere in its output geometry: the
row of the classifier head that scores class $y$ has a norm and a bias, and it sits at some
set of distances from the rows of competing classes. Crowded rows correspond to classes that
are easily confused and therefore hard to cover. These quantities are available for a class
with zero calibration examples, since they are determined by the model's parameters and do
not involve any sample. Our goal is to use them to assign such a class a threshold that is
specific to it, while keeping the marginal coverage guarantee that split conformal prediction
provides. \Cref{fig:teaser} shows one such class.

\paragraph{Our approach.} Predicted Class Correction (PCC) proceeds in three steps
(\cref{alg:pcc}). On the
classes that do have calibration data we observe the offset $\delta_y$ between the per-class
conformal quantile and the marginal one, and fit a ridge regression of $\delta_y$ on the
descriptors $\varphi(y)$. We then apply the fitted regression to every class, including those
with no data, and shrink its output by a factor $\lambda$ chosen on the labelled classes only.
Finally we refit one scalar offset so that marginal coverage returns to its target. The
shrinkage step does substantial work here, since worst-class coverage is governed by the
largest error in the predicted offset, which its average error does not control, so applying
the raw prediction at full strength is harmful (\cref{tab:ablation}).

\paragraph{Summary of findings.} We evaluate on ImageNet \cite{russakovsky2015imagenet}, two
long-tailed datasets \cite{garcin2021plantnet,vanhorn2018inaturalist}, three backbones, four
score functions and $15$ corruption types \cite{hendrycks2019benchmarking}, for a total of
$412$ configurations. (i) On ImageNet classes with zero calibration examples, PCC improves
worst-class coverage by $0.058$ $[0.006,0.110]$ at matched average set size and is the only one
of nine methods whose interval excludes zero; six of the eight baselines change the
statistic by exactly zero (\cref{tab:competitors}). (ii) Few classes need to be labelled for this to work: holding out
$90\%$ of classes, which leaves $100$ labelled classes out of $1000$, still yields $0.052$
with a confidence interval excluding zero, and our protocol's floor is $p+2$ labelled classes
for $p$ descriptors (\cref{fig:sweeps}b). (iii) Below roughly $35$ evaluation examples per
class, an oracle that observes the test labels cannot improve worst-class coverage either, so
what disappears there is the room to improve and not the ability to take it
(\cref{fig:sweeps}a). This one axis accounts for all four settings in which the method
initially appeared to fail, and restricting evaluation above the threshold recovers a
significant improvement on Pl@ntNet-300K, though not on iNaturalist-2018.

\section{Related Work}
\label{sec:related}

\paragraph{Class-conditional conformal prediction.} Classwise conformal prediction, an
instance of Mondrian conformal prediction \cite{vovk2005algorithmic,vovk2012conditional}, calibrates an
independent quantile per class and thereby attains class-conditional coverage
\cite{sadinle2019least,guan2022prediction}, at the price of
very large sets when classes are thin \cite{sadinle2019least,lofstrom2015bias}. Two lines of
work soften that price by sharing information between classes: \textsc{Clustered} CP groups
classes by an embedding of their score distributions and calibrates one quantile per cluster
\cite{ding2023class}, while rank-calibrated CP couples a per-class quantile with a per-class
rank budget \cite{shi2024rank}. Most recently, \citet{ding2026conformal} embed each class by
the quantiles of its own score distribution and compute a kernel-weighted quantile that
borrows from nearby classes, alongside a prevalence-adjusted score function. Few-shot conformal
prediction \cite{fisch2021few} also targets thin calibration data, by meta-learning the score
from auxiliary tasks that have plenty of labelled examples, whereas here the class has none.

All of these embeddings are built from the labelled calibration scores of the class being
embedded. A class with none therefore has no position of its own, and the fall-backs differ only in
where it is placed: at the extreme of the range, in a null cluster, or at an infinite
threshold. Descriptors read from the classifier's weights have no such gap, which is the
difference \cref{tab:competitors} measures.

\paragraph{Conformal score functions.} Normalised and locally adaptive conformal prediction
estimate a per-instance difficulty from the features and rescale the score by it
\cite{lei2014distribution,lei2018distribution}, the per-instance analogue of a per-class correction. A parallel
line designs the score itself to trade set size against conditional coverage, including APS
\cite{romano2020classification}, RAPS \cite{angelopoulos2021uncertainty} and SAPS
\cite{huang2024conformal}, all of which appear in \hyperref[tab:depth]{Tab.~S4}.

\paragraph{Prevalence as the source of the correction.} A second family leaves the per-class
quantile alone and reshapes the score using class frequency. \citet{ding2026conformal} divide
the softmax probability by an estimate of the class prior, which is optimal for macro-coverage
among procedures of a given expected set size, and blend the classwise and marginal quantiles
with a user-chosen weight. \citet{liu2026tail} add a rank-dependent penalty to the scores of
head classes, with a smooth reweighting in their sTACP variant, and prove a smaller head-tail
coverage gap. Both remain defined for a class with no calibration examples, and both draw their
per-class signal from one scalar. Prevalence alone is significantly worse than geometry
(\cref{tab:ablation}) and is beaten by a random projection of the same dimension
(\cref{tab:competitors}).

\paragraph{Group-conditional coverage.} A third family asks for coverage conditional on group
membership instead of per class. \citet{bairaktari2025kandinsky} give the general form,
allowing overlapping and fractional groups defined jointly on covariates and labels, and
\citet{gibbs2025conditional} give conditional guarantees over any finite-dimensional class of
covariate shifts or groups. Groups
drawn in descriptor space would trade a per-class claim for a guarantee about neighbourhoods,
which we return to in \cref{sec:discussion}.

\paragraph{Conformal prediction under distribution shift.} Coverage can be restored under
covariate shift by importance-weighting the calibration scores \cite{tibshirani2019conformal},
under label shift with a known or estimated class prior \cite{podkopaev2021distribution}, in
general by weighting the calibration scores \cite{barber2023conformal}, or over time by adapting
the miscoverage level from observed errors \cite{gibbs2021adaptive}.
All need a density ratio, a prior or online label feedback, which our setting does not provide, so
we read our corrupted-input experiment as a stress test and make no coverage claim there
(\hyperref[app:imagenetc]{Sec.~L}).

\paragraph{Long-tailed recognition.} Reweighting and logit adjustment improve accuracy under
class imbalance \cite{cui2019class,cao2019learning,menon2021long}, as does retraining the
classifier head separately from the representation \cite{kang2020decoupling}, and fine-grained
biodiversity datasets have made the regime concrete
\cite{garcin2021plantnet,vanhorn2018inaturalist}. Our concern is complementary
and post-hoc: given whatever classifier such methods produce, how should its prediction sets
be calibrated for the classes that training consumed. The same separation holds for conformal
training \cite{stutz2022learning,einbinder2022training}, which differentiates through the
conformal predictor while fitting the model: the scores improve, but a class-conditional
guarantee still needs calibration examples of that class afterwards.

\section{Predicted Class Correction}
\label{sec:method}

\subsection{Setup}
\label{sec:setup}

Let $(X_i,Y_i)_{i=1}^{n}$ be calibration pairs and $(X_{n+1},Y_{n+1})$ a test pair, all
exchangeable, with labels in $\mathcal{Y}=\{1,\dots,K\}$, and use split (inductive) conformal prediction
\cite{papadopoulos2002inductive,lei2018distribution}. A conformal score
$s:\mathcal{X}\times\mathcal{Y}\to\mathbb{R}$ is larger when a pair conforms less; we use
$s(x,y)=1-\hat p(y\mid x)$, the LAC score of \citet{sadinle2019least}, and report APS
\cite{romano2020classification}, RAPS \cite{angelopoulos2021uncertainty} and SAPS
\cite{huang2024conformal} in \hyperref[tab:depth]{Tab.~S4}. Write $n_y$ for the number of calibration
examples of class $y$, and let
\begin{equation}
\hat q = \mathrm{Quantile}\Big(\{s(X_i,Y_i)\}_{i=1}^n;\ \tfrac{\lceil (n+1)(1-\alpha)\rceil}{n}\Big)
\label{eq:qglobal}
\end{equation}
be the marginal conformal quantile. For a set-valued procedure $\mathcal{C}$ we write
\begin{equation}
\begin{aligned}
\mathrm{CondCov}(\mathcal{C},y) &= \mathbb{P}\bigl(Y\in\mathcal{C}(X)\mid Y=y\bigr), \\
W(\mathcal{C}) &= \min_{y\in\mathcal{Y}} \mathrm{CondCov}(\mathcal{C},y),
\end{aligned}
\label{eq:worst}
\end{equation}
so $W$ is the worst-class coverage this paper reports. Classwise conformal prediction replaces
$\hat q$ by a per-class quantile $\hat q_y$ computed from the $n_y$ examples of class $y$;
equivalently it applies a per-class \emph{offset}
\begin{equation}
\delta_y \;=\; \hat q_y - \hat q ,
\label{eq:delta}
\end{equation}
which is positive for classes that need a looser threshold. When $n_y=0$, $\hat q_y$ and
hence $\delta_y$ are undefined.

\paragraph{Held-out classes.} We study the zero-example case directly. A fraction of classes
is designated as \emph{held out} and all of their calibration rows are deleted, so that
$n_y$ is exactly zero for them. Results are then reported in two tables that
are never averaged together, one for the classes that kept their data and one for the held-out
classes, since a method that wins on the held-out classes by sacrificing the others has not
solved the problem and a single pooled number would conceal that.

\paragraph{Choice of intervention.} A natural alternative is to improve $s(x,y)$ itself, by a
better model, better probability calibration \cite{guo2017calibration}, or aggregation over
transformed copies of the input \cite{shanmugam2021better}. Any such change reshapes the score for every class at once, while class-conditional
coverage is set by where each class's threshold falls in its own score distribution. Reducing
the spread of per-class coverage therefore needs a degree of freedom per class, which
$\delta_y$ supplies and a score-level change does not; \hyperref[app:score-level]{Sec.~O} measures this.

\paragraph{Matched set size.} Coverage can always be bought with larger sets, so comparing
coverage at unmatched set size compares nothing. Writing
$\mathcal{C}_\delta(x)=\{y: s(x,y)\le \hat q+\delta_y\}$ for the sets induced by an offset
vector $\delta\in\mathbb{R}^{K}$, every method we report is shifted by the scalar
\begin{equation}
\sigma(\delta) \;=\;
\inf\Bigl\{\,\sigma\in\mathbb{R} \;:\;
\tfrac{1}{N}\textstyle\sum_{i=1}^{N}\bigl|\mathcal{C}_{\delta+\sigma}(X_i)\bigr| \;\ge\; T \Bigr\}
\label{eq:sigma}
\end{equation}
for a common target size $T$. We call $\bar\delta = \delta + \sigma(\delta)$ the
\emph{matched} offset, and every comparison in \cref{sec:exp} is read at it. Solving
\cref{eq:sigma} on the calibration rows instead, so that no test input enters the matching,
gives the same $+0.058$ to four decimals (\hyperref[app:sigmacalib]{Sec.~V}). \Citet{dhillon2024expected} characterise the
expected size as a function of the threshold; the factor it needs to translate score measure
into label measure has no closed form for LAC, which is why \cref{eq:sigma} is solved
numerically.

\subsection{Calibration-label-free class descriptors}
\label{sec:phi}

For a linear classifier head with weight rows $w_y\in\mathbb{R}^d$ and biases $b_y$, we form a
descriptor vector $\varphi(y)\in\mathbb{R}^{p}$, $p=9$, from
\begin{enumerate}[itemsep=1pt,topsep=2pt,leftmargin=*]
\item $\|w_y\|_2$ and $b_y$, the scale and offset of the class logit;
\item cosine distances from $w_y$ to its $k$ nearest competing rows for
      $k\in\{5,10,50\}$, the distance to the nearest row, the spread of that
      neighbourhood, and the cosine distance to the mean row, all of which measure how
      crowded the class is;
\item $\log$ of the class's row count in the score dump: proportional to its
      training count on ImageNet, and its calibration-split size on the long-tailed sets.
\end{enumerate}
Row norms are known to track class frequency under imbalanced training
\cite{kang2020decoupling}, and the arrangement of the rows is the structure studied under neural
collapse \cite{papyan2020prevalence}.
The eight geometric features are functions of the model's parameters alone, so they are
defined for a class with no examples of any kind and carry no sampling noise. Prevalence is
counted from the training labels, which are available whenever the model is; the descriptors
are free of \emph{calibration} labels, though not of labels altogether. No result depends
on it: removing prevalence costs $0.011$ in worst-class coverage without reaching significance
($p=0.30$), while prevalence alone is significantly worse than the full set
(\cref{tab:ablation}). The
feature that most closely resembles the regression target is held out, so no reported number
can be attributed to it.

\subsection{The correction}
\label{sec:correction}

\paragraph{Step 1: regress.} On the classes with at least $n_{\mathrm{cal}}$ calibration
examples, a count fixed in advance, we observe $\delta_y$ from
\cref{eq:delta} and fit a ridge regression \cite{hoerl1970ridge} $g_\theta:\mathbb{R}^p\to\mathbb{R}$ to predict it
from $\varphi(y)$. There is no network and no training loop: $g_\theta$ is a closed-form fit
over at most $K$ points with $p\le 15$ features, which costs milliseconds
(\hyperref[tab:runtime]{Tab.~S19}).

\paragraph{Step 2: shrink.} The raw prediction is applied with a scalar $\lambda$,
\begin{equation}
\tilde\delta_y \;=\; \lambda\cdot g_\theta(\varphi(y)) ,
\label{eq:shrink}
\end{equation}
with $\lambda$ chosen over a fixed grid $\Lambda$ by maximising worst-class coverage at
matched size on the labelled classes only,
\begin{equation}
\hat\lambda \;=\; \operatorname*{arg\,max}_{\lambda\in\Lambda}\;
\widehat W_{\mathcal{Y}_{\mathrm{lab}}}\bigl(\mathcal{C}_{\overline{\lambda g_\theta(\varphi)}}\bigr),
\label{eq:lambda}
\end{equation}
where $\widehat W_{\mathcal{Y}_{\mathrm{lab}}}$ is the empirical version of \cref{eq:worst}
restricted to $\mathcal{Y}_{\mathrm{lab}}$. The restriction matters, because choosing
$\lambda$ with help from the held-out classes would let a free parameter manufacture the very
result the held-out table is meant to test, and it would void \cref{prop:marginal}.

Shrinkage is what makes the method usable, and \cref{prop:sup} says why;
\hyperref[fig:method]{Fig.~S1} shows the fit it is applied to. Worst-class
coverage is controlled by the \emph{largest} error in the offsets, a quantity mean squared
error does not bound, so a predictor with a respectable $R^2$ can still be harmful at full
strength: $\lambda=1$ costs $0.584$ relative to the selected $\lambda\approx0.1$, while
$\lambda=0$ recovers the marginal threshold exactly (\cref{tab:ablation}).

\begin{proposition}[Worst-class coverage is a sup-norm functional]
\label{prop:sup}
Let $F_y(t)=\mathbb{P}(s(X,Y)\le t\mid Y=y)$ and suppose each $F_y$ is $L$-Lipschitz. For any
two offset vectors $\delta,\delta'\in\mathbb{R}^{K}$ with matched versions
$\bar\delta,\bar\delta'$ from \cref{eq:sigma},
\begin{equation}
\bigl|\,W(\mathcal{C}_{\bar\delta}) - W(\mathcal{C}_{\bar\delta'})\,\bigr|
\;\le\; L\,\bigl\|\bar\delta-\bar\delta'\bigr\|_{\infty}.
\label{eq:supbound}
\end{equation}
\end{proposition}

\begin{proof}
$\mathrm{CondCov}(\mathcal{C}_{\bar\delta},y)=F_y(\hat q+\bar\delta_y)$, so by Lipschitzness every $y$ obeys
\begin{equation*}
|F_y(\hat q+\bar\delta_y)-F_y(\hat q+\bar\delta'_y)|
\;\le\; L\,\|\bar\delta-\bar\delta'\|_\infty .
\end{equation*} Two functions of $y$ that agree
pointwise to within $\varepsilon$ have minima that agree to within $\varepsilon$, and
\cref{eq:worst} is such a minimum.
\end{proof}

Taking $\delta'$ to be the oracle offsets bounds our worst-class coverage below by the
oracle's, less $L$ times the sup-norm error. The bound is in the sup norm, and a ridge fit
controls a mean, so a single badly predicted class can cost as much as many well predicted ones
save. Shrinking controls that sup norm, but the sup norm is not what selects $\hat\lambda$:
swept over the grid it is smallest at $\lambda=0.7$ while worst-class coverage peaks at
$0.1$, and the bound is two orders of magnitude looser than the gap it bounds (\hyperref[app:lipschitz]{Sec.~Z}). What
$\hat\lambda$ limits instead is the global downshift \cref{eq:sigma} must apply to pay for a
wider correction, which reaches $-0.24$ at $\lambda=1$ and which this bound does not see.

\paragraph{Accuracy required of $g_\theta$.} We have no theory saying when $\varphi$ predicts
$\delta$ well, and the method works exactly as far as the regression does; a quadratic or
kernel head does not change that (\hyperref[app:heads]{Sec.~S}). Across our three
datasets the two line up monotonically. The cross-validated error of $g_\theta$ is $0.036$ on
ImageNet, $0.072$ on Pl@ntNet-300K and $0.127$ on iNaturalist-2018, against worst-class
improvements of $0.058$, $0.019$ and $0.000$. Because \cref{eq:lambda} can return $\hat\lambda=0$, a poor regression is
abstained from, which is what happens on iNaturalist-2018 in every split (\cref{sec:depth}).

\paragraph{Step 3: recalibrate.} Finally one scalar is refit, on labelled-class rows only, so
that marginal coverage returns to $1-\alpha$. Write $\tilde s(x,y)=s(x,y)-\tilde\delta_y$ for
the corrected score and let
\begin{equation}
c \;=\; \mathrm{Quantile}\Bigl(\{\tilde s(X_i,Y_i)\}_{i\in\mathcal{I}_{\mathrm{rec}}};\
\tfrac{\lceil (n_{\mathrm{rec}}+1)(1-\alpha)\rceil}{n_{\mathrm{rec}}}\Bigr) \;-\; \hat q ,
\label{eq:c}
\end{equation}
a \emph{residual} relative to $\hat q$. The prediction set is
\begin{equation}
\mathcal{C}(x) \;=\; \bigl\{\, y : s(x,y) \le \hat q + \tilde\delta_y + c \,\bigr\},
\label{eq:set}
\end{equation}
and $\hat q + c$ is by construction the conformal quantile of $\tilde s$ on
$\mathcal{I}_{\mathrm{rec}}$, so \cref{eq:set} is exactly the split conformal set for
$\tilde s$, with $\hat q$ entering the threshold once.

\begin{algorithm}[t]
\small
\caption{Predicted Class Correction}
\label{alg:pcc}
\begin{algorithmic}[1]
\Require score $s$, level $\alpha$, target size $T$, grid $\Lambda$, count $n_{\mathrm{cal}}$
\Require disjoint calibration rows $\mathcal{I}_{\mathrm{fit}},\mathcal{I}_{\mathrm{rec}}$;
         labelled classes $\mathcal{Y}_{\mathrm{lab}}$; descriptors $\varphi$
\State $\hat q \gets$ marginal conformal quantile \eqref{eq:qglobal} on
       $\mathcal{I}_{\mathrm{fit}}$
\For{$y\in\mathcal{Y}_{\mathrm{lab}}$ with $n_y\ge n_{\mathrm{cal}}$}
  \State $\delta_y \gets \hat q_y - \hat q$ \Comment{observed offset \eqref{eq:delta}}
\EndFor
\State $g_\theta \gets$ ridge fit of $\delta$ on $\varphi$ over those classes
\State $\hat\lambda \gets$ \eqref{eq:lambda}, on $\mathcal{Y}_{\mathrm{lab}}$ only
\State $\tilde\delta_y \gets \hat\lambda\, g_\theta(\varphi(y))$ for \emph{every}
       $y\in\mathcal{Y}$ \Comment{including $n_y=0$}
\State $c \gets$ conformal quantile of
       $\{s(X_i,Y_i)-\tilde\delta_{Y_i}\}_{i\in\mathcal{I}_{\mathrm{rec}}}$, minus $\hat q$
       \Comment{residual \eqref{eq:c}}
\State \Return $\mathcal{C}(x)=\{y : s(x,y)\le \hat q+\tilde\delta_y+c\}$
\end{algorithmic}
\end{algorithm}

\subsection{Data splitting}
\label{sec:splits}

Every quantity above is estimated, so it matters which rows and which classes each one may
see. Classes are partitioned into labelled and held-out, and independently the rows of every
class into a calibration and an evaluation part; \hyperref[tab:splits]{Tab.~S12} in the supplementary states the
combination each quantity is fitted on. Two invariants hold throughout: nothing is fitted on evaluation rows,
and nothing is fitted on a held-out class. The numbers reported for held-out classes are
therefore free of selection effects from $\lambda$, $n_{\mathrm{cal}}$ and $c$, while those for
labelled classes are not, which is the second reason the two tables are kept apart.

\subsection{Coverage guarantee}
\label{sec:guarantee}

Let $\mathcal{I}_{\mathrm{fit}}$ and $\mathcal{I}_{\mathrm{rec}}$ be disjoint sets of
calibration rows, of sizes $n_{\mathrm{fit}}$ and $n_{\mathrm{rec}}$.

\begin{proposition}[Marginal coverage is preserved]
\label{prop:marginal}
Let $(X_i,Y_i)_{i\in\mathcal{I}_{\mathrm{rec}}}$ and $(X_{n+1},Y_{n+1})$ be exchangeable, and
let $\tilde\delta_\cdot$ be any function of $\varphi(\cdot)$ and of
$(X_i,Y_i)_{i\in\mathcal{I}_{\mathrm{fit}}}$. Let $c$ be given by \cref{eq:c}, so that
$\hat q + c$ is the conformal quantile at level
$\lceil (n_{\mathrm{rec}}+1)(1-\alpha)\rceil/n_{\mathrm{rec}}$ of the shifted scores
$\{\tilde s(X_i,Y_i)\}_{i\in\mathcal{I}_{\mathrm{rec}}}$. Then
$\mathbb{P}\bigl(Y_{n+1}\in\mathcal{C}(X_{n+1})\bigr)\ge 1-\alpha$.
\end{proposition}

\begin{proof}[Proof sketch]
Define $\tilde s(x,y)=s(x,y)-\tilde\delta_y$. Since $\tilde\delta_\cdot$ is a function of
$\varphi(\cdot)$ alone, and $\varphi$ is computed from the model's weights and does not involve
these calibration rows, $\tilde s$ is a \emph{fixed} score function with respect to the
sample. Exchangeability is closed under applying a fixed map pointwise, so
$\tilde s(X_1,Y_1),\dots,\tilde s(X_{n+1},Y_{n+1})$ are exchangeable and the standard split
conformal argument applies verbatim to $\tilde s$. By \cref{eq:c} the threshold
$\hat q + c$ \emph{is} that argument's quantile, so membership in $\mathcal{C}$ is exactly the
event it bounds. Full proof in \hyperref[app:proof]{Sec.~A}.
\end{proof}

The proposition holds because only \emph{one} degree of freedom is calibrated on
$\mathcal{I}_{\mathrm{rec}}$; $\tilde\delta_y$ may be arbitrarily complex, since it merely
redefines the score.

\paragraph{Reuse of calibration rows.} The premise names two disjoint row sets.
\Cref{tab:competitors} provides them, since \textsc{Pcc}-split holds a quarter of the
calibration rows back for $c$ alone; the sweeps that follow do not, so there $\tilde\delta$ and
$c$ are dependent and the proposition does not apply to them as stated. The cost is measured: at
the thresholds as emitted over ten ImageNet splits, \textsc{Pcc} covers $0.8994$ where the
marginal threshold covers $0.9012$ on the same rows, a deficit of $0.0018$ and roughly twenty
times the finite-sample slack $1/(n+1)$. The split variant removes it, and at this protocol the
two agree on worst-class coverage to four decimals (\hyperref[app:split]{Sec.~P}).

\paragraph{Scope of the claim.} \Cref{prop:marginal} covers the distribution the
recalibration rows are drawn from. Since $c$ is fit on labelled-class rows, it does not reach a
test point whose class has $n_y=0$, and no procedure can make it: such a class contributes no
row to calibrate with. What we claim there is empirical, and $\lambda$ and $n_{\mathrm{cal}}$
are chosen on the labelled label space, which is why the two tables are kept apart (\hyperref[app:scope]{Sec.~B}).

\section{Experiments}
\label{sec:exp}

\paragraph{Protocol.} Unless stated otherwise we use ImageNet-1k
\cite{russakovsky2015imagenet} scores from a SimCLRv2 encoder with a linear probe
\cite{chen2020big}, descriptors from a supervised ResNet-50 head \cite{he2016deep} so that
$\varphi$ is exogenous to the scoring model, $\alpha=0.10$, $30\%$ of classes held out with
all of their calibration rows deleted, $n_{\mathrm{cal}}=25$, and five random splits.
Intervals are $95\%$ confidence intervals over splits, and every comparison is at matched
average set size (\cref{sec:setup}).

\paragraph{Metrics.} \hyperref[tab:metrics]{Tab.~S2} in the supplementary defines every quantity. The headline
statistic is worst-class coverage, since one badly covered class is what a class-conditional
guarantee is meant to prevent. It can only be estimated when classes have enough evaluation
examples, so we fix a threshold of $30$ per class in advance: above it we report per-class
worst coverage (regime~A), below it a prevalence-binned surrogate (regime~B). Every result
states its regime, and \cref{sec:depth} shows why that matters.

\paragraph{Oracle ceiling.} To give the numbers a scale we report an oracle that observes the
test labels, forms the exact per-class offset, then selects its own shrinkage. It is
unattainable and is not a method; it says how much of the available room a method took.
Percentages of it are ratios of averages: the ceiling moves by $0.10$ between splits, so the
fraction taken ranges from $2\%$ to $75\%$ (\hyperref[app:oracle]{Sec.~F}).

% ---------------------------------------------------------------- competitors
\begin{table}[t]
\centering
\small
\setlength{\tabcolsep}{4pt}
\begin{tabular}{lcc}
\toprule
Method & $\Delta$ worst-class & Undefined \\
\midrule
\textsc{Standard}        & $+0.000\ [\ \ 0.000,\ 0.000]$ & $0/300$ \\
\textsc{Classwise}       & $+0.000\ [\ \ 0.000,\ 0.000]$ & $\mathbf{300/300}$ \\
\textsc{Clustered}       & $+0.000\ [\ \ 0.000,\ 0.000]$ & $0/300$ \\
\textsc{Rc3p}            & $+0.000\ [\ \ 0.000,\ 0.000]$ & $\mathbf{300/300}$ \\
\textsc{Fuzzy}, quantile & $+0.000\ [\ \ 0.000,\ 0.000]$ & $0/300$ \\
\textsc{Fuzzy}, rarity   & $+0.000\ [-0.000,+0.001]$ & $0/300$ \\
\textsc{Fuzzy}, random   & $+0.016\ [-0.009,+0.040]$ & $0/300$ \\
\textsc{Interp-Q}        & $+0.000\ [\ \ 0.000,\ 0.000]$ & $0/300$ \\
\midrule
\textbf{PCC (ours)} & $\mathbf{+0.058}\ [+0.006,+0.110]$ & $\mathbf{0/300}$ \\
\midrule
\textit{Oracle ceiling} & \textit{+0.140} & --- \\
\bottomrule
\end{tabular}
\caption{Classes with zero calibration examples. \textsc{Standard}
\cite{vovk2005algorithmic}, \textsc{Classwise} \cite{sadinle2019least}, \textsc{Clustered}
\cite{ding2023class}, \textsc{Rc3p} \cite{shi2024rank}, and \textsc{Fuzzy} and
\textsc{Interp-Q} \cite{ding2026conformal}. Change in worst-class coverage over the marginal
threshold, at matched average set size ($1.221$ for every row), on $300$ held-out ImageNet
classes, five splits at the protocol of \cref{sec:exp}. PCC is in the split variant of
\cref{sec:guarantee}, so the row carrying the headline is the one \cref{prop:marginal} covers;
the reuse variant gives the same figure to four decimals (\hyperref[app:split]{Sec.~P}), as does solving the size match on calibration rows rather than evaluation rows (\hyperref[app:sigmacalib]{Sec.~V}). PCC is the only method whose
interval excludes zero. ``Undefined'' counts classes a method could not give a threshold to,
before the pre-registered fallback. Baseline numbers come from their released code, and for
\textsc{Fuzzy} CP the kernel bandwidth is swept with the best kept, which favours the
baseline; the winning bandwidth is not stable across splits.}
\label{tab:competitors}
\end{table}

\subsection{Classes with no calibration examples}
\label{sec:comp}

PCC improves worst-class coverage on the held-out classes by $0.058$ $[0.006,0.110]$, $42\%$ of
the oracle ceiling, and is the only entry of nine whose interval excludes zero
(\cref{tab:competitors}): at this protocol the question is not which method gains most but
which gains anything. A group-conditional baseline on the same descriptors --- $k$-means
on $\varphi$, one conformal quantile per cluster --- reaches $+0.054$ $[-0.004,+0.112]$, with a
paired difference of $+0.004$ $[-0.014,+0.022]$: what the regression buys over a grouping is
unresolved at this depth (\hyperref[app:methods]{Sec.~C}). Percentile and CVaR summaries of the same tail are in \hyperref[app:tail]{Sec.~Y}.

Two of the zeros are not measurements: \textsc{Classwise} CP and \textsc{Rc3p} return an
infinite threshold for all $300$ classes, so their entry records the pre-registered fallback.
The other six are measurements, and \textsc{Interp-Q} is the instructive one: it caps an
undefined classwise quantile at one, so a held-out class receives the widest threshold in the
label space and the others pay for it once size is matched back down. Swept over its
interpolation weight, it selects zero on every split, at which point it is \textsc{Standard} CP.

Prevalence is a weak signal here. \textsc{Fuzzy} CP reaches $0.000$ with its rarity
projection and $0.016$ with a \emph{random} one, so widening thresholds for data-poor classes in
an arbitrary direction already buys more than prevalence does; our own ablation agrees, since
prevalence-only descriptors drive $\lambda$ to zero in all ten splits (\cref{tab:ablation}).
The descriptors of \cref{sec:phi} buy $3.6$ times what a random direction buys.

The gain is paid for. On the first split the correction lifts the worst class from $0.658$ to
$0.691$ and moves macro-coverage from $0.904$ to $0.900$, but because set size is fixed,
$11\%$ of classes gain coverage while $43\%$ lose a mean of $0.004$
(\hyperref[fig:coverage]{Fig.~S2} and \hyperref[tab:whohelped]{Tab.~S10}).

\subsection{Ablations}
\label{sec:ablations}

\begin{table}[t]
\centering
\small
\setlength{\tabcolsep}{4pt}
\begin{tabular}{lccc}
\toprule
Ablation & $\Delta$ vs.\ full & $d_z$ & Holm \\
\midrule
$\lambda=0$ (no correction)      & $+0.032$ & $1.30$ & \checkmark \\
$\lambda=1$ (unshrunk)           & $+0.584$ & $3.99$ & \checkmark \\
distance descriptors only        & $+0.039$ & $1.49$ & \checkmark \\
prevalence descriptor only       & $+0.032$ & $1.30$ & \checkmark \\
without prevalence               & $+0.011$ & $0.33$ & --- \\
$n^\star$ rule replaced          & $+0.000$ & --- & --- \\
no marginal recalibration        & $+0.000$ & --- & --- \\
\bottomrule
\end{tabular}
\caption{Ablations on ImageNet at $\alpha=0.10$, ten splits. Positive means the full method is
better. $d_z$ is Cohen's paired effect size \cite{cohen1988statistical}; the last column marks Holm-corrected significance
at $0.05$ \cite{holm1979simple}. The last two rows are identical to the full method at matched
set size by construction (\hyperref[app:invisible]{Sec.~I}).}
\label{tab:ablation}
\end{table}

Shrinkage carries the method. Both endpoints of the $\lambda$ grid are significantly worse
than the selected interior value, and the unshrunk correction costs $0.584$ in worst-class
coverage (\cref{tab:ablation}). Among descriptors, distances alone and prevalence alone are
both significantly worse than the full set, while removing prevalence is not ($p=0.30$), so
the geometry carries the signal.

The last two rows are exactly zero because a constant added to every threshold is cancelled
exactly by the size-matching shift, and marginal recalibration is such a constant. At the thresholds as emitted it moves
marginal coverage from $0.8963$ to $0.8989$ against a target of $0.900$, at a cost of $0.31$
in set size. \hyperref[app:invisible]{Sec.~I} gives that view in full.

% ---------------------------------------------------------------- sweeps
\begin{figure*}[t]
\centering
\includegraphics[width=\linewidth]{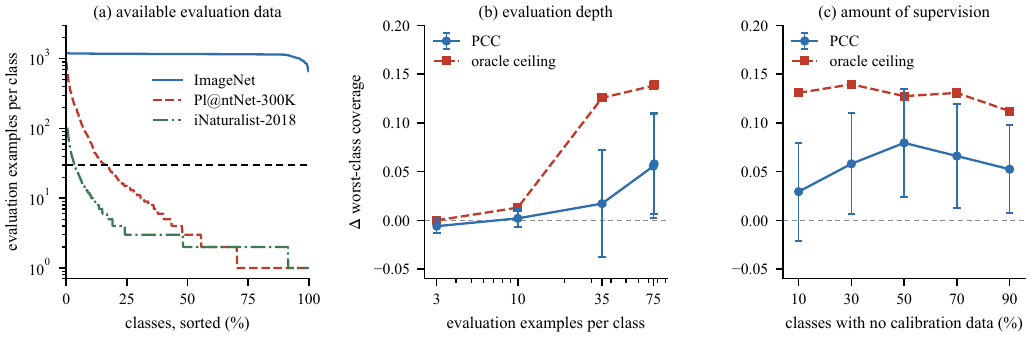}
\caption{(a) Evaluation examples available per class, sorted, for the three datasets. The
dashed line is the $30$-example threshold below which a per-class coverage estimate is withheld
(\hyperref[tab:metrics]{Tab.~S2}); $83\%$ of Pl@ntNet-300K classes and $97\%$ of iNaturalist-2018 classes
fall below it, against none of ImageNet's. (b) Capping evaluation examples per class inside one
dataset with everything else fixed: below roughly $35$ the oracle ceiling itself falls to zero,
so no method, including one that observes the test labels, can be shown to improve worst-class
coverage. The settings where PCC appeared to fail sit at $35$ (additional backbones), $3$
(Pl@ntNet-300K) and $2$ (iNaturalist-2018). (c) Sweeping how many classes have no calibration
data at all; the $x$-values correspond to $900$, $700$, $500$, $300$ and $100$ labelled classes
out of $1000$. Bars are $95\%$ confidence intervals over five splits.}
\label{fig:sweeps}
\end{figure*}

\subsection{The role of evaluation depth}
\label{sec:depth}

Applied without modification, PCC fails on two long-tailed datasets and two additional
backbones. All four lie on one axis, with a second cause beneath it for one, and finding the
axis meant discarding two explanations we had held. Calibration depth is not the cause. Capping calibration examples inside ImageNet
gives a dose-response from $0.033$ at ten per class to $0.060$ at one hundred, with
every interval excluding zero (\hyperref[tab:depth]{Tab.~S4}), and ten per class is far thinner than the
settings that fail. Descriptors endogenous to the scoring model are not the cause either:
replacing each backbone's own head with a fixed exogenous one leaves every interval containing
zero (\hyperref[app:backbones]{Sec.~J}).

How much evaluation data each dataset actually has is \cref{fig:sweeps}a. On ImageNet every
class carries about $1168$ examples; on Pl@ntNet-300K the median class has three and on
iNaturalist-2018 two, so $83\%$ and $97\%$ of their classes respectively fall below the
threshold at which we report a per-class estimate at all. \Cref{fig:sweeps}b caps the number of
\emph{evaluation} examples per class inside ImageNet instead. The
effect is monotone and its interval crosses zero between $35$ and $75$. More importantly the
oracle ceiling collapses on the same axis, reaching zero by three examples per class, where
even a method holding the test labels could not improve worst-class coverage. The axis
predicts the four failures quantitatively: the additional backbones sit at $35$ and measure
$0.017$, which is what the sweep gives at $35$; Pl@ntNet-300K sits at three and iNaturalist-2018 at two, where the sweep
is at or below zero.

That explanation is falsifiable, and restricting evaluation to classes above the threshold
confirms it on Pl@ntNet-300K. Of its $1081$ classes, $98$ have at least $75$ evaluation
examples and $34$ of those are held out; over that label space the improvement is $0.019$
$[0.010,0.028]$, five splits of five. At the weaker threshold of $35$ it does not reach
significance, $0.020$ $[-0.011,0.052]$, which is what \cref{fig:sweeps}b predicts at the edge
of measurability. The apparent failure was an artefact of measurement, and the price of
removing it is that the claim covers under a tenth of the label space (\hyperref[tab:longtail]{Tab.~S7}).

iNaturalist-2018 stays at zero even after restriction, and there the cause is different.
$\lambda$ is selected as $0$ in every split, so the method returns the marginal threshold
untouched. The shrinkage curve is not flat, which would mean $\hat\delta$ is degenerate; it
decreases monotonically, so $\hat\delta$ carries signal that points the wrong way. The
regression is correspondingly poor, with cross-validated error $0.127$ against $0.036$ on
ImageNet. Eight times as many classes in a descriptor space of the same dimension is the
obvious explanation and richer descriptors the obvious remedy, but we have tested neither
(\hyperref[app:longtail]{Sec.~K}). What shrinkage delivers here is that a $3.5\times$ worse regression
costs nothing, because the method detects that it cannot help and abstains.

% ---------------------------------------------------------------- held-out

\subsection{Amount of supervision required}
\label{sec:supervision}

Sweeping the held-out fraction from $10\%$ to $90\%$, the effect peaks at $50\%$ with $0.080$,
or $62\%$ of the ceiling, and remains significant at $90\%$, where only $100$ classes retain
labels, at $0.052$ $[0.007,0.098]$ (\cref{fig:sweeps}b). Four of the five points exclude zero;
the exception is $10\%$, where only $100$ classes are held out so the worst-class statistic
over them is noisier and the room is smaller. We refuse to fit $g_\theta$ on fewer than $p+2$
labelled classes, eleven here. Ridge is defined below that, since the penalty makes the system
solvable; we decline it because there the fit follows the penalty rather than the data.

\subsection{Score functions and miscoverage level}
\label{sec:scores}

The correction survives every score but is strongest with LAC and attenuates for the adaptive
ones (\hyperref[tab:depth]{Tab.~S4}), plausibly because APS, RAPS and SAPS already reshape the score using
rank or cumulative mass, absorbing part of the class-difficulty signal $\delta_y$ captures. At
$\alpha=0.01$ the effect is exactly zero, consistent with the offsets becoming
indistinguishable from noise in the extreme tail.

PAS \cite{ding2026conformal} appears because it is the other published construction that stays
defined at $n_y=0$. It does not remove the need for a per-class correction, though the reading
is not a clean win either: our improvement on PAS scores is $0.066$ with an interval containing
zero, while the ceiling rises to $0.234$, so PAS opens more room than LAC and we take a smaller
share of it.

\paragraph{Corrupted inputs.}
Calibrating on clean images and evaluating on ImageNet-C \cite{hendrycks2019benchmarking},
over $15$ corruptions at two severities, breaks exchangeability, so we make no coverage claim
here and neither can any method compared. At contrast severity $5$, marginal coverage is
$0.053$ and $75.7\%$ of test points receive an empty set. Pooled over the $90$ conditions PCC
gives $0.003$ $[0.001,0.005]$, an order of magnitude below the clean setting
(\hyperref[app:imagenetc]{Sec.~L}).

\paragraph{Computational cost.} PCC trains nothing: a closed-form ridge fit over at most $K$
points with $p=9$ features, on descriptors read from the weights, so no images, gradients or
GPU are involved (\hyperref[app:runtime]{Sec.~U}).

\section{Discussion}
\label{sec:discussion}

A per-class conformal correction can be predicted from calibration-label-free descriptors and
applied to classes with no calibration examples at all. On ImageNet it takes $42\%$ of an
oracle ceiling there and is the only one of nine methods whose improvement is distinguishable
from none, two of the eight baselines being undefined for all such classes, and it needs labels for as few as
a tenth of the classes.
Marginal coverage is preserved because a single scalar is calibrated
(\cref{prop:marginal}).

\paragraph{Limitations.} The claim at $n_y=0$ is empirical, for the reason
\cref{sec:correction} gives. Our zero-calibration classes are made by deleting the
calibration rows of well-represented ones, so the head row supplying $\varphi(y)$ was fitted
with their training data, and the improvement rests entirely on that row: perturbed outward by
the magnitude Pl@ntNet-300K's rare classes actually exhibit, it falls to $-0.014$ (\hyperref[app:noise]{Sec.~T}).
Real rare rows do differ measurably, though neither long-tailed outcome tracks that
difference. And $\varphi$ reads a linear classifier head, which a model
scoring by embedding similarity does not have. The improvement attenuates for adaptive
scores and disappears at $\alpha=0.01$. Transfer across architectures remains
untested, since the additional backbones leave $35$ evaluation examples per class, which
\cref{fig:sweeps}a places below the depth at which anything is detectable.
Five splits is what ImageNet's validation set supports at this depth, and a grouping of the
same descriptors reaches $+0.054$, which this evaluation cannot separate from PCC (\hyperref[app:methods]{Sec.~C}).
The depth at which the ceiling collapses was measured on one dataset at one $\alpha$. Finally, every table uses a
random per-class subsample of the score dump, adopted for memory reasons
(\hyperref[app:rowbudget]{Sec.~Q}); sweeping it over a factor of three moves the improvement by less than
the interval of any single setting, so the headline is not an artefact of it, but we cannot
say what happens on the full dump.

\paragraph{Evaluating class-conditional coverage.} The observation we expect to outlast this
method is that worst-class coverage becomes unmeasurable before it becomes unachievable, in
the sense that the evaluation set runs out of power before the method runs out of room. Below
roughly $35$ evaluation examples per class the oracle ceiling is zero, so a null result on such
a slice says little about the method that produced it, and \cref{fig:sweeps}a shows how much of
each long-tailed benchmark sits there. Published comparisons on them may report a property of
the evaluation set where they intend a difference between methods. We state the measurability
regime beside every number, and three of our own explanations were discarded before we
understood this (\cref{sec:depth}).

\paragraph{Future work.} The most direct next step is a larger evaluation set for the rare
classes themselves, since the restriction in \hyperref[tab:longtail]{Tab.~S7} keeps under a tenth of the
label space. A second question is why $\varphi$ carries signal on Pl@ntNet-300K and not on
iNaturalist-2018; both fit $g_\theta$ from under eighty measurable classes, so the label
space is not the axis (Secs.~\hyperref[app:longtail]{K} and~\hyperref[app:heads]{S}). The
most interesting direction is a formal statement to replace the empirical one. No procedure can
bound coverage for an individual class with $n_y=0$, but group-conditional conformal prediction
can bound it for a neighbourhood of classes, and the framework of
\citet{bairaktari2025kandinsky} accepts overlapping, fractional groups defined jointly on
covariates and labels, which is the shape our descriptors have. Groups drawn in
$\varphi$-space would give a guarantee about regions of the classifier's output geometry:
weaker than a per-class statement, and much stronger than none.

{
    \small
    \bibliographystyle{ieeenat_fullname}
    \bibliography{main}

@String(PAMI = {IEEE Trans. Pattern Anal. Mach. Intell.})

@String(IJCV = {Int. J. Comput. Vis.})

@String(CVPR= {IEEE Conf. Comput. Vis. Pattern Recog.})

@String(ICCV= {Int. Conf. Comput. Vis.})

@String(NIPS= {Adv. Neural Inform. Process. Syst.})

@String(ICLR = {Int. Conf. Learn. Represent.})

@String(AAAI = {AAAI})

@String(PAMI  = {IEEE TPAMI})

@String(IJCV  = {IJCV})

@String(CVPR  = {CVPR})

@String(ICCV  = {ICCV})

@String(NIPS  = {NeurIPS})

@String(ICLR  = {ICLR})

@book{vovk2005algorithmic,
  title     = {Algorithmic Learning in a Random World},
  author    = {Vovk, Vladimir and Gammerman, Alexander and Shafer, Glenn},
  publisher = {Springer},
  address   = {New York},
  year      = {2005},
  doi       = {10.1007/b106715},
}

@article{angelopoulos2023gentle,
  title   = {Conformal Prediction: A Gentle Introduction},
  author  = {Angelopoulos, Anastasios N. and Bates, Stephen},
  journal = {Foundations and Trends in Machine Learning},
  volume  = {16},
  number  = {4},
  pages   = {494--591},
  year    = {2023},
  doi     = {10.1561/2200000101},
}

@article{sadinle2019least,
  title   = {Least Ambiguous Set-Valued Classifiers with Bounded Error Levels},
  author  = {Sadinle, Mauricio and Lei, Jing and Wasserman, Larry},
  journal = {Journal of the American Statistical Association},
  volume  = {114},
  number  = {525},
  pages   = {223--234},
  year    = {2019},
  doi     = {10.1080/01621459.2017.1395341},
  note    = {arXiv:1609.00451},
}

@article{lei2014distribution,
  title   = {Distribution-Free Prediction Bands for Non-parametric Regression},
  author  = {Lei, Jing and Wasserman, Larry},
  journal = {Journal of the Royal Statistical Society: Series B},
  volume  = {76},
  number  = {1},
  pages   = {71--96},
  year    = {2014},
  doi     = {10.1111/rssb.12021},
}

@inproceedings{romano2020classification,
  title     = {Classification with Valid and Adaptive Coverage},
  author    = {Romano, Yaniv and Sesia, Matteo and Candès, Emmanuel J.},
  booktitle = NIPS,
  volume    = {33},
  pages     = {3581--3591},
  year      = {2020},
}

@inproceedings{angelopoulos2021uncertainty,
  title     = {Uncertainty Sets for Image Classifiers Using Conformal Prediction},
  author    = {Angelopoulos, Anastasios N. and Bates, Stephen and Malik, Jitendra and Jordan, Michael I.},
  booktitle = ICLR,
  year      = {2021},
  note      = {arXiv:2009.14193},
}

@inproceedings{huang2024conformal,
  title     = {Conformal Prediction for Deep Classifier via Label Ranking},
  author    = {Huang, Jianguo and Xi, Huajun and Zhang, Linjun and Yao, Huaxiu and Qiu, Yue and Wei, Hongxin},
  booktitle = {Int. Conf. Mach. Learn.},
  volume    = {235},
  series    = {Proceedings of Machine Learning Research},
  pages     = {20331--20347},
  year      = {2024},
  note      = {arXiv:2310.06430},
}

@inproceedings{ding2023class,
  title     = {Class-Conditional Conformal Prediction with Many Classes},
  author    = {Ding, Tiffany and Angelopoulos, Anastasios N. and Bates, Stephen and Jordan, Michael I. and Tibshirani, Ryan J.},
  booktitle = NIPS,
  volume    = {36},
  pages     = {64555--64576},
  year      = {2023},
  note      = {arXiv:2306.09335},
}

@inproceedings{ding2026conformal,
  title     = {Conformal Prediction for Long-Tailed Classification},
  author    = {Ding, Tiffany and Fermanian, Jean-Baptiste and Salmon, Joseph},
  booktitle = ICLR,
  year      = {2026},
  note      = {arXiv:2507.06867},
}

@inproceedings{shi2024rank,
  title     = {Conformal Prediction for Class-wise Coverage via Augmented Label Rank Calibration},
  author    = {Shi, Yuanjie and Ghosh, Subhankar and Belkhouja, Taha and Doppa, Janardhan Rao and Yan, Yan},
  booktitle = NIPS,
  volume    = {37},
  year      = {2024},
  note      = {arXiv:2406.06818},
}

@inproceedings{tibshirani2019conformal,
  title     = {Conformal Prediction Under Covariate Shift},
  author    = {Tibshirani, Ryan J. and Barber, Rina Foygel and Candes, Emmanuel J. and Ramdas, Aaditya},
  booktitle = NIPS,
  volume    = {32},
  pages     = {2530--2540},
  year      = {2019},
}

@inproceedings{gibbs2021adaptive,
  title     = {Adaptive Conformal Inference Under Distribution Shift},
  author    = {Gibbs, Isaac and Candes, Emmanuel J.},
  booktitle = NIPS,
  volume    = {34},
  pages     = {1660--1672},
  year      = {2021},
}

@inproceedings{garcin2021plantnet,
  title     = {Pl@ntNet-300K: A Plant Image Dataset with High Label Ambiguity and a Long-Tailed Distribution},
  author    = {Garcin, Camille and Joly, Alexis and Bonnet, Pierre and Affouard, Antoine and Lombardo, Jean-Christophe and Chouet, Mathias and Servajean, Maximilien and Lorieul, Titouan and Salmon, Joseph},
  booktitle = {NeurIPS Datasets and Benchmarks Track},
  year      = {2021},
}

@inproceedings{vanhorn2018inaturalist,
  title     = {The iNaturalist Species Classification and Detection Dataset},
  author    = {Van Horn, Grant and Mac Aodha, Oisin and Song, Yang and Cui, Yin and Sun, Chen and Shepard, Alex and Adam, Hartwig and Perona, Pietro and Belongie, Serge},
  booktitle = CVPR,
  pages     = {8769--8778},
  year      = {2018},
}

@article{russakovsky2015imagenet,
  title   = {ImageNet Large Scale Visual Recognition Challenge},
  author  = {Russakovsky, Olga and Deng, Jia and Su, Hao and Krause, Jonathan and Satheesh, Sanjeev and Ma, Sean and Huang, Zhiheng and Karpathy, Andrej and Khosla, Aditya and Bernstein, Michael and Berg, Alexander C. and Fei-Fei, Li},
  journal = IJCV,
  volume  = {115},
  number  = {3},
  pages   = {211--252},
  year    = {2015},
  doi     = {10.1007/s11263-015-0816-y},
}

@inproceedings{hendrycks2019benchmarking,
  title     = {Benchmarking Neural Network Robustness to Common Corruptions and Perturbations},
  author    = {Hendrycks, Dan and Dietterich, Thomas},
  booktitle = ICLR,
  year      = {2019},
}

@inproceedings{he2016deep,
  title     = {Deep Residual Learning for Image Recognition},
  author    = {He, Kaiming and Zhang, Xiangyu and Ren, Shaoqing and Sun, Jian},
  booktitle = CVPR,
  pages     = {770--778},
  year      = {2016},
}

@inproceedings{dosovitskiy2021image,
  title     = {An Image is Worth 16x16 Words: Transformers for Image Recognition at Scale},
  author    = {Dosovitskiy, Alexey and Beyer, Lucas and Kolesnikov, Alexander and Weissenborn, Dirk and Zhai, Xiaohua and Unterthiner, Thomas and Dehghani, Mostafa and Minderer, Matthias and Heigold, Georg and Gelly, Sylvain and Uszkoreit, Jakob and Houlsby, Neil},
  booktitle = ICLR,
  year      = {2021},
}

@inproceedings{liu2022convnet,
  title     = {A ConvNet for the 2020s},
  author    = {Liu, Zhuang and Mao, Hanzi and Wu, Chao-Yuan and Feichtenhofer, Christoph and Darrell, Trevor and Xie, Saining},
  booktitle = CVPR,
  pages     = {11976--11986},
  year      = {2022},
}

@inproceedings{chen2020big,
  title     = {Big Self-Supervised Models are Strong Semi-Supervised Learners},
  author    = {Chen, Ting and Kornblith, Simon and Swersky, Kevin and Norouzi, Mohammad and Hinton, Geoffrey},
  booktitle = NIPS,
  volume    = {33},
  pages     = {22243--22255},
  year      = {2020},
}

@inproceedings{menon2021long,
  title     = {Long-Tail Learning via Logit Adjustment},
  author    = {Menon, Aditya Krishna and Jayasumana, Sadeep and Rawat, Ankit Singh and Jain, Himanshu and Veit, Andreas and Kumar, Sanjiv},
  booktitle = ICLR,
  year      = {2021},
}

@article{holm1979simple,
  title   = {A Simple Sequentially Rejective Multiple Test Procedure},
  author  = {Holm, Sture},
  journal = {Scandinavian Journal of Statistics},
  volume  = {6},
  number  = {2},
  pages   = {65--70},
  year    = {1979},
}

@article{lofstrom2015bias,
  title   = {Bias Reduction through Conditional Conformal Prediction},
  author  = {L{\"o}fstr{\"o}m, Tuve and Bostr{\"o}m, Henrik and Linusson, Henrik and Johansson, Ulf},
  journal = {Intelligent Data Analysis},
  volume  = {19},
  number  = {6},
  pages   = {1355--1375},
  year    = {2015},
  doi     = {10.3233/IDA-150775},
}

@inproceedings{liu2026tail,
  author    = {Liu, Shuqi and Huang, Jianguo and Ong, Luke},
  title     = {Conformal Prediction Meets Long-tail Classification},
  booktitle = AAAI,
  volume    = {40},
  number    = {28},
  pages     = {23828--23836},
  year      = {2026},
  doi       = {10.1609/aaai.v40i28.39558},
  note      = {arXiv:2508.11345},
}

@inproceedings{dhillon2024expected,
  author    = {Dhillon, Guneet S. and Deligiannidis, George and Rainforth, Tom},
  title     = {On the Expected Size of Conformal Prediction Sets},
  booktitle = {AISTATS},
  volume    = {238},
  series    = {Proceedings of Machine Learning Research},
  pages     = {1549--1557},
  year      = {2024},
  note      = {arXiv:2306.07254},
}

@inproceedings{bairaktari2025kandinsky,
  author    = {Bairaktari, Konstantina and Wu, Jiayun and Wu, Zhiwei Steven},
  title     = {Kandinsky Conformal Prediction: Beyond Class- and Covariate-Conditional
               Coverage},
  booktitle = {ICML},
  year      = {2025},
  note      = {arXiv:2502.17264},
}

@inproceedings{stutz2022learning,
  title     = {Learning Optimal Conformal Classifiers},
  author    = {Stutz, David and Dvijotham, Krishnamurthy and Cemgil, Ali Taylan and Doucet, Arnaud},
  booktitle = ICLR,
  year      = {2022},
}

@inproceedings{einbinder2022training,
  title     = {Training Uncertainty-Aware Classifiers with Conformalized Deep Learning},
  author    = {Einbinder, Bat-Sheva and Romano, Yaniv and Sesia, Matteo and Zhou, Yanfei},
  booktitle = NIPS,
  volume    = {35},
  year      = {2022},
}

@article{shafer2008tutorial,
  title   = {A Tutorial on Conformal Prediction},
  author  = {Shafer, Glenn and Vovk, Vladimir},
  journal = {Journal of Machine Learning Research},
  volume  = {9},
  pages   = {371--421},
  year    = {2008},
  note    = {arXiv:0706.3188},
}

@inproceedings{papadopoulos2002inductive,
  title     = {Inductive Confidence Machines for Regression},
  author    = {Papadopoulos, Harris and Proedrou, Kostas and Vovk, Volodya and Gammerman, Alex},
  booktitle = {Machine Learning: ECML 2002},
  series    = {Lecture Notes in Computer Science},
  pages     = {345--356},
  publisher = {Springer},
  year      = {2002},
  doi       = {10.1007/3-540-36755-1_29},
}

@article{lei2018distribution,
  title   = {Distribution-Free Predictive Inference for Regression},
  author  = {Lei, Jing and G'Sell, Max and Rinaldo, Alessandro and Tibshirani, Ryan J. and Wasserman, Larry},
  journal = {Journal of the American Statistical Association},
  volume  = {113},
  number  = {523},
  pages   = {1094--1111},
  year    = {2018},
  doi     = {10.1080/01621459.2017.1307116},
}

@inproceedings{vovk2012conditional,
  title     = {Conditional Validity of Inductive Conformal Predictors},
  author    = {Vovk, Vladimir},
  booktitle = {Asian Conference on Machine Learning (ACML)},
  series    = {Proceedings of Machine Learning Research},
  volume    = {25},
  year      = {2012},
  note      = {arXiv:1209.2673},
}

@article{barber2021limits,
  title   = {The Limits of Distribution-Free Conditional Predictive Inference},
  author  = {Barber, Rina Foygel and Cand{\`e}s, Emmanuel J. and Ramdas, Aaditya and Tibshirani, Ryan J.},
  journal = {Information and Inference: A Journal of the IMA},
  volume  = {10},
  number  = {2},
  pages   = {455--482},
  year    = {2021},
  doi     = {10.1093/imaiai/iaaa017},
}

@article{gibbs2025conditional,
  title   = {Conformal Prediction with Conditional Guarantees},
  author  = {Gibbs, Isaac and Cherian, John J. and Cand{\`e}s, Emmanuel J.},
  journal = {Journal of the Royal Statistical Society Series B: Statistical Methodology},
  volume  = {87},
  number  = {4},
  pages   = {1100--1126},
  year    = {2025},
  doi     = {10.1093/jrsssb/qkaf008},
}

@article{guan2022prediction,
  title   = {Prediction and Outlier Detection in Classification Problems},
  author  = {Guan, Leying and Tibshirani, Robert},
  journal = {Journal of the Royal Statistical Society Series B: Statistical Methodology},
  volume  = {84},
  number  = {2},
  pages   = {524--546},
  year    = {2022},
  doi     = {10.1111/rssb.12443},
}

@inproceedings{fisch2021few,
  title     = {Few-Shot Conformal Prediction with Auxiliary Tasks},
  author    = {Fisch, Adam and Schuster, Tal and Jaakkola, Tommi and Barzilay, Regina},
  booktitle = {ICML},
  year      = {2021},
  note      = {arXiv:2102.08898},
}

@inproceedings{podkopaev2021distribution,
  title     = {Distribution-Free Uncertainty Quantification for Classification under Label Shift},
  author    = {Podkopaev, Aleksandr and Ramdas, Aaditya},
  booktitle = {Conference on Uncertainty in Artificial Intelligence (UAI)},
  year      = {2021},
  note      = {arXiv:2103.03323},
}

@article{barber2023conformal,
  title   = {Conformal Prediction beyond Exchangeability},
  author  = {Barber, Rina Foygel and Cand{\`e}s, Emmanuel J. and Ramdas, Aaditya and Tibshirani, Ryan J.},
  journal = {The Annals of Statistics},
  volume  = {51},
  number  = {2},
  pages   = {816--845},
  year    = {2023},
  doi     = {10.1214/23-AOS2276},
}

@article{zhang2023deep,
  title   = {Deep Long-Tailed Learning: A Survey},
  author  = {Zhang, Yifan and Kang, Bingyi and Hooi, Bryan and Yan, Shuicheng and Feng, Jiashi},
  journal = PAMI,
  volume  = {45},
  number  = {9},
  pages   = {10795--10816},
  year    = {2023},
  doi     = {10.1109/TPAMI.2023.3268118},
}

@inproceedings{kang2020decoupling,
  title     = {Decoupling Representation and Classifier for Long-Tailed Recognition},
  author    = {Kang, Bingyi and Xie, Saining and Rohrbach, Marcus and Yan, Zhicheng and Gordo, Albert and Feng, Jiashi and Kalantidis, Yannis},
  booktitle = ICLR,
  year      = {2020},
  note      = {arXiv:1910.09217},
}

@inproceedings{cui2019class,
  title     = {Class-Balanced Loss Based on Effective Number of Samples},
  author    = {Cui, Yin and Jia, Menglin and Lin, Tsung-Yi and Song, Yang and Belongie, Serge},
  booktitle = CVPR,
  pages     = {9260--9269},
  year      = {2019},
  doi       = {10.1109/CVPR.2019.00949},
}

@inproceedings{cao2019learning,
  title     = {Learning Imbalanced Datasets with Label-Distribution-Aware Margin Loss},
  author    = {Cao, Kaidi and Wei, Colin and Gaidon, Adrien and Arechiga, Nikos and Ma, Tengyu},
  booktitle = NIPS,
  volume    = {32},
  year      = {2019},
  note      = {arXiv:1906.07413},
}

@article{papyan2020prevalence,
  title   = {Prevalence of Neural Collapse during the Terminal Phase of Deep Learning Training},
  author  = {Papyan, Vardan and Han, X. Y. and Donoho, David L.},
  journal = {Proceedings of the National Academy of Sciences},
  volume  = {117},
  number  = {40},
  pages   = {24652--24663},
  year    = {2020},
  doi     = {10.1073/pnas.2015509117},
}

@inproceedings{guo2017calibration,
  title     = {On Calibration of Modern Neural Networks},
  author    = {Guo, Chuan and Pleiss, Geoff and Sun, Yu and Weinberger, Kilian Q.},
  booktitle = {ICML},
  year      = {2017},
  note      = {arXiv:1706.04599},
}

@inproceedings{shanmugam2021better,
  title     = {Better Aggregation in Test-Time Augmentation},
  author    = {Shanmugam, Divya and Blalock, Davis and Balakrishnan, Guha and Guttag, John},
  booktitle = ICCV,
  year      = {2021},
  note      = {arXiv:2011.11156},
}

@article{hoerl1970ridge,
  title   = {Ridge Regression: Biased Estimation for Nonorthogonal Problems},
  author  = {Hoerl, Arthur E. and Kennard, Robert W.},
  journal = {Technometrics},
  volume  = {12},
  number  = {1},
  pages   = {55--67},
  year    = {1970},
  doi     = {10.1080/00401706.1970.10488634},
}

@book{cohen1988statistical,
  title     = {Statistical Power Analysis for the Behavioral Sciences},
  author    = {Cohen, Jacob},
  edition   = {2nd},
  publisher = {Lawrence Erlbaum Associates},
  address   = {Hillsdale, NJ},
  year      = {1988},
}
}

\clearpage
\setcounter{figure}{0}
\setcounter{table}{0}
\renewcommand{\thefigure}{S\arabic{figure}}
\renewcommand{\thetable}{S\arabic{table}}

\begin{center}
{\Large\bf Supplementary Material}
\end{center}
\vspace{4pt}

\noindent
Numbers carrying a letter or an \textsf{S} refer to this part: \cref{app:proof} and
\cref{tab:methods} are here. Plain numbers refer to the main paper above, whose sections,
figures, tables and propositions are numbered from one.

\appendix
\section{Proof of the marginal coverage guarantee}
\label{app:proof}

Write $\tilde s(x,y)=s(x,y)-\tilde\delta_y$ for the corrected score, and recall from
\hyperref[sec:splits]{Sec.~3.4} that $\tilde\delta_\cdot$ is built from $\varphi(\cdot)$ and from the rows
$\mathcal{I}_{\mathrm{fit}}$, while $c$ is computed from the disjoint rows
$\mathcal{I}_{\mathrm{rec}}$.

\emph{Step 1: $\tilde s$ is a fixed function.} $\varphi$ is a function of the classifier's
weights, and $g_\theta$, $\lambda$ and $n_{\mathrm{cal}}$ are functions of
$(X_i,Y_i)_{i\in\mathcal{I}_{\mathrm{fit}}}$. None of them reads
$(X_i,Y_i)_{i\in\mathcal{I}_{\mathrm{rec}}}$ or the test point. Conditionally on
$\mathcal{I}_{\mathrm{fit}}$, therefore, $\tilde s$ is a deterministic map
$\mathcal{X}\times\mathcal{Y}\to\mathbb{R}$, fixed before any of the rows that will be used to
calibrate $c$ are seen. All statements below are conditional on
$\mathcal{I}_{\mathrm{fit}}$; since the bound obtained is uniform in it, the unconditional
bound follows by averaging.

\emph{Step 2: exchangeability survives.} If $(Z_1,\dots,Z_m)$ is exchangeable and $f$ is a
fixed measurable map, then $(f(Z_1),\dots,f(Z_m))$ is exchangeable, because for any permutation
$\pi$ the vector $(f(Z_{\pi(1)}),\dots,f(Z_{\pi(m)}))$ is the image under the same product map
of a vector equal in distribution to $(Z_1,\dots,Z_m)$. Applying this with
$Z_i=(X_i,Y_i)$ for $i\in\mathcal{I}_{\mathrm{rec}}\cup\{n+1\}$ and $f=\tilde s$ gives that
$\{\tilde s(X_i,Y_i)\}_{i\in\mathcal{I}_{\mathrm{rec}}}\cup\{\tilde s(X_{n+1},Y_{n+1})\}$ is
exchangeable.

\emph{Step 3: the split conformal bound.} For exchangeable real random variables
$V_1,\dots,V_{n_{\mathrm{rec}}},V_{n+1}$ and $\kappa$ the empirical quantile of
$V_1,\dots,V_{n_{\mathrm{rec}}}$ at level
$\lceil (n_{\mathrm{rec}}+1)(1-\alpha)\rceil/n_{\mathrm{rec}}$, the standard argument
\cite{vovk2005algorithmic,angelopoulos2023gentle} gives
$\mathbb{P}(V_{n+1}\le \kappa)\ge 1-\alpha$. Take $V_i=\tilde s(X_i,Y_i)$.

\emph{Step 4: the event is set membership.} The recalibration scalar is defined as
$c=\kappa-\hat q$, so $\hat q + c = \kappa$ exactly. By \hyperref[eq:set]{Eq.~9},
$y\in\mathcal{C}(x)\iff s(x,y)\le\hat q+\tilde\delta_y+c\iff\tilde s(x,y)\le \kappa$, which is
the event Step 3 bounds. $\square$

Stating $c$ as a residual matters here rather than being a matter of taste. Had $c$ been the
quantile $\kappa$ itself while the prediction set still added $\hat q$, the threshold would
have been $\hat q + \kappa$ and $\hat q$ would enter twice; the bound would still hold, since
the set only grows, but the procedure would no longer be the one the implementation runs, whose
scalar solves for marginal coverage on top of the corrected thresholds and takes the value
$0.0023$ on the primary configuration.

\paragraph{What the proof does not give.} Only one scalar is calibrated on
$\mathcal{I}_{\mathrm{rec}}$, so $\tilde\delta_y$ may be arbitrarily complex without affecting
the argument; but for the same reason the conclusion is a statement about the marginal
distribution alone. Nothing here bounds
$\mathbb{P}(Y_{n+1}\in\mathcal{C}(X_{n+1})\mid Y_{n+1}=y)$ for any particular $y$, and for a
class with $n_y=0$ no such bound is available to any procedure. The disjointness in Step 1 is
also load-bearing: an earlier version of our code selected $\lambda$ over all classes including
the held-out ones, which does not merely bias an estimate but removes the premise, and it moved
the reported effect from $0.3$ to $0.5$ on the diagnostic we caught it with.

\section{What the guarantee covers, and what it cannot}
\label{app:scope}

Two claims run through this paper and they cover different things. Confusing them would let a
reader hear something we do not say, so they are separated here.

\paragraph{(A) Marginal validity.} \hyperref[prop:marginal]{Prop.~2} bounds
$\mathbb{P}(Y_{n+1}\in\mathcal{C}(X_{n+1}))$ from below by $1-\alpha$, where the probability is
over the joint draw of the recalibration rows and the test point, and the premise is that the
two are exchangeable. The scalar $c$ is fit on rows of the labelled classes, so the distribution
the guarantee is stated over is the one those rows are drawn from. A test point belonging to a
class with $n_y=0$ is not exchangeable with them --- that class contributes no row --- and the
bound does not reach it. This is not a gap in the proof but a limit on what any procedure can
promise: a class that supplies no calibration example cannot enter the exchangeability argument
that a distribution-free bound is built from.

\paragraph{(B) Empirical improvement at $n_y=0$.} What we claim for those classes is measured,
not proved: worst-class coverage at matched average set size, reported in \hyperref[tab:competitors]{Tab.~1}
over five splits with confidence intervals. Everything in \hyperref[sec:exp]{Sec.~4} about held-out classes
is of this kind. We do not claim, and the reader should not read, that PCC guarantees
class-conditional coverage for a class it has never seen a label of.

\paragraph{The flow, and which parts (A) touches.} Calibration rows of the labelled classes are
partitioned into $\mathcal{I}_{\mathrm{fit}}$ and $\mathcal{I}_{\mathrm{rec}}$. Everything
data-dependent except $c$ is estimated on $\mathcal{I}_{\mathrm{fit}}$: $\hat q$, the observed
offsets, $g_\theta$, $\hat\lambda$ and $n_{\mathrm{cal}}$. For $\hat\lambda$ that covers both
quantities \hyperref[eq:lambda]{Eq.~6} optimises over, not just the grid: the worst-class statistic it maximises and
the size-matching shift it is read at are both computed on $\mathcal{I}_{\mathrm{fit}}$ rows of
the labelled classes. Scoring $\hat\lambda$ on evaluation rows instead would not be deployable,
since a deployed system has no evaluation labels. Only $c$ reads
$\mathcal{I}_{\mathrm{rec}}$, and only $\mathcal{I}_{\mathrm{rec}}$ enters (A). The evaluation
rows and the held-out classes enter neither; they are read once, at the end, to report (B).
\Cref{tab:splits} states this per quantity.

\paragraph{Why the construction is a stress test and not a deployment.} No deployment labels a
class ``held out''. One calibrates on whatever labelled data exists and the guarantee covers the
population that data came from, which under-represents exactly the rare classes one cares about.
Deleting the calibration rows of chosen classes makes that under-representation total and
therefore visible: it is the worst case of a condition real deployments meet partially. It also
means the classes we test on were trained with abundant data, which \cref{app:noise} measures
the consequence of.

\section{What each method gives a class with no data}
\label{app:methods}

\Cref{tab:methods} lists the per-class threshold rules compared in this paper and what each
one produces when $n_y=0$. Three of them stay defined there without per-class labels, and
\hyperref[sec:comp]{Sec.~4.1} reports what that is worth.

% ---------------------------------------------------------------- method summary
\begin{table}[ht]
\centering
\small
\setlength{\tabcolsep}{3pt}
\begin{tabular}{@{}lll@{}}
\toprule
Method & $q_y$ built from & If $n_y=0$ \\
\midrule
\textsc{Standard} \cite{vovk2005algorithmic}  & shared quantile   & shared \\
\textsc{Classwise} \cite{sadinle2019least}    & own quantile      & $\infty$ \\
\textsc{Clustered} \cite{ding2023class}       & cluster quantile  & shared \\
\textsc{Rc3p} \cite{shi2024rank}              & quantile + rank   & $\infty$ \\
\textsc{Fuzzy} \cite{ding2026conformal}       & kernel weights    & shared \\
\textsc{Pas} \cite{ding2026conformal}         & shared quantile   & shared \\
\textsc{Interp-Q} \cite{ding2026conformal}    & blend of the two  & capped \\
\textsc{Tacp} \cite{liu2026tail}              & head penalty      & by group \\
\midrule
\textbf{PCC} (ours) & $g_\theta(\varphi(y))$ & \textbf{predicted} \\
\bottomrule
\end{tabular}
\caption{Methods that give a class its own conformal threshold. All of them target at least
marginal coverage; the columns that separate them are what happens to a class with no labelled
calibration examples, and how specific to that class the threshold it receives can be. Three
methods do produce a defined threshold there without per-class labels: \textsc{Pas} reshapes the
score using prevalence, \textsc{Interp-Q} caps the undefined classwise quantile at one before
blending, and \textsc{Tacp} penalises the score of head classes, which gives one degree of
freedom per prevalence group instead of one per class. Behaviours were obtained by running the
authors' released implementations on the same held-out classes, except for \textsc{Tacp}, whose
authors release no code; \hyperref[tab:competitors]{Tab.~1} reports the coverage that results.}
\label{tab:methods}
\end{table}

\paragraph{A group-conditional baseline on the same descriptors.} None of the published rows
above reads $\varphi$, so none of them isolates what the regression contributes. The baseline
that does is the obvious one: $k$-means on the standardised descriptors, one conformal quantile
per cluster calibrated from that cluster's labelled members, and a held-out class inheriting
its cluster's quantile. It is defined at $n_y=0$ for the same reason PCC is --- cluster
membership is read from the weights --- and it goes through the identical size matching and
statistics, so the two are directly comparable.

Over the five splits of the primary configuration it reaches $+0.054$ $[-0.004,+0.112]$
against PCC's $+0.058$ $[+0.006,+0.110]$, and the paired difference is $+0.004$
$[-0.014,+0.022]$. Only PCC's interval excludes zero, but the two methods are not separated by
this experiment, and we report it that way rather than as a win.

Two asymmetries run in the baseline's favour and should be stated with the number. Its cluster
count is swept over $\{5,10,20,50,100\}$ and the best is kept \emph{on the statistic it is
then scored by}, which is the same courtesy \hyperref[tab:competitors]{Tab.~1} extends to 	extsc{Fuzzy}
CP's bandwidth; PCC's $\hat\lambda$ is $0.10$ on every split and is selected on the labelled
classes alone, never on the classes it is scored on. And the winning $k$ is unstable --- $100$,
$10$, $50$, $20$, $20$ across the five splits --- so the baseline as deployed, with $k$ fixed
in advance, would do worse than the number above. What the comparison establishes is that at
this evaluation depth a grouping in $\varphi$-space captures most of what the regression on
$\varphi$ captures, and that separating them needs more evaluation rows per class than
ImageNet's validation set provides --- the same limit \cref{app:longtail} runs into from the
other direction.

\section{Evaluation metrics}
\label{app:metrics}

\Cref{tab:metrics} defines every quantity reported in the paper. The two regimes are the part
that matters: above a median of $30$ evaluation examples per class we report per-class worst
coverage, and below it a prevalence-binned surrogate, with the threshold fixed before any run
(\cref{app:prereg}).

% ---------------------------------------------------------------- metrics
\begin{table}[ht]
\centering
\small
\setlength{\tabcolsep}{3pt}
\begin{tabular}{@{}ll@{}}
\toprule
Metric & Definition \\
\midrule
$\hat c_y$
  & $\frac{1}{|J_y|}\sum_{i\in J_y}\mathbf{1}\{Y_i\in\mathcal{C}(X_i)\}$ \\[2pt]
Worst-class (A)
  & $\min_{y}\hat c_y$ \\[2pt]
Bin-worst (B)
  & $\min_{b}\hat c^{(b)}$, prevalence bins \\[2pt]
$\Delta$ worst
  & worst-class minus marginal \\[2pt]
Below target
  & $\frac{1}{|\mathcal{Y}|}\sum_y\mathbf{1}\{\hat c_y<1-\alpha\}$ \\[2pt]
SSCV
  & $\max_{b}|\hat c^{(b)}-(1-\alpha)|$, strata \\[2pt]
Worst-slab
  & $\min$ mean $\hat c_y$ over windows \\[2pt]
Set size
  & $\frac{1}{N}\sum_i|\mathcal{C}(X_i)|$ \\[2pt]
Empty rate
  & $\frac{1}{N}\sum_i\mathbf{1}\{|\mathcal{C}(X_i)|=0\}$ \\
\bottomrule
\end{tabular}
\caption{Evaluation metrics. $\hat c_y$ is the empirical coverage of class $y$, and regimes A and B are the per-class and prevalence-binned cases. $J_y$ is the set of test indices with label $y$ and $N$ the
number of test points. Bins and strata are pooled groups, used where a per-class estimate is
too noisy to report: prevalence bins require $200$ test rows each and set-size strata $30$.
The empty-set rate is included because an average set size below one is only possible when a
substantial share of test points receive nothing, a regime in which coverage and set size both
merely look low.}
\label{tab:metrics}
\end{table}

\section{Pre-registration and the fallback policy}
\label{app:prereg}

Two documents were frozen before any experiment in this paper was run, and both are in the
released repository. We summarise what they fix, because several of the choices would otherwise
look like decisions made after seeing results.

\paragraph{Measurability thresholds.} A survey of the score dumps, not of our memory, recorded
the median number of evaluation examples per class: $1168$ on ImageNet, $100$ on CIFAR-100,
$3$ on Pl@ntNet-300K and $2$ on iNaturalist-2018. With two examples a class's coverage can only
be $0$, $\tfrac12$ or $1$, so a minimum over $8142$ such classes is almost surely $0$ for every
procedure including the oracle, and differences between methods in that column carry no
information. The rule fixed in advance is therefore a median of at least $30$ evaluation
examples per class for per-class statistics (regime~A) and at least $200$ pooled evaluation rows
per prevalence bin otherwise (regime~B). Below the threshold, per-class statistics are still
computed but are written under a \texttt{withheld\_unmeasurable} key and may not enter a table
or decide a verdict. The same rule was later applied to $\lambda$ \emph{selection} as well as to
reporting, which was an omission we found on 2026-08-15 and fixed.

\paragraph{Behaviour of each baseline at $n_y=0$.} What each baseline emits for a class with
no calibration examples was written down before its released code was read: an infinite
threshold for classwise CP and for RC3P, a fall-back to the marginal threshold for Clustered CP
through its null cluster, and a kernel-weighted quantile over labelled classes for the method of
\citet{ding2026conformal}. All four predictions were then confirmed against the code. The
fall-back policy applied to an infinite threshold, namely substituting $\hat q$, was fixed at
the same time; it is what produces the zeros in \hyperref[tab:competitors]{Tab.~1}, and \hyperref[sec:comp]{Sec.~4.1}
says so in the text so that a zero is not read as a measurement.

\paragraph{Changes made after seeing results.} Three, all recorded. The oracle ceiling was
redefined from the unshrunk to the shrunk per-class quantile after PCC exceeded the unshrunk
version, for the reason given in \cref{app:oracle}. The calibration fraction was raised from
$0.50$ to $0.70$ after two points of the calibration-depth sweep turned out not to be binding on
a thin slice. The $\lambda$ ablation was extended from five splits to ten after it missed
Holm-corrected significance at five ($p=0.0094$ against a threshold of $0.0083$); it passes at
ten, and we report the number of splits for every table.

\section{The oracle ceiling, shrunk and unshrunk}
\label{app:oracle}

The ceiling reported in the main text is the \emph{best shrinkage} of a perfect per-class
offset: we form $\delta_y$ from the test labels and then select $\lambda$ for it exactly as
the method does. An earlier version used the unshrunk per-class quantiles, and PCC exceeded
it --- which is how we learned that the unshrunk version is not a bound. It is $\lambda=1$
applied to a perfect offset, and $\lambda=1$ is harmful for worst-class coverage
(\hyperref[tab:ablation]{Tab.~2}).

Across all $228$ configurations of the main sweep the two definitions agree in
$194$ cases, because the oracle's own shrinkage selects
$\lambda_{\text{oracle}}=1$ there; on the primary ImageNet arm they coincide at $+0.140$, so
the correction does not move that headline number. Where they differ, they differ by a lot,
and always in the same direction. \Cref{tab:oracle} lists the five largest gaps: the unshrunk
``ceiling'' falls as low as $-0.657$, far \emph{below} the marginal threshold it is supposed
to bound from above.

\begin{table}[ht]
\centering
\small
\setlength{\tabcolsep}{4pt}
\begin{tabular}{lccc}
\toprule
Configuration & Shrunk & Unshrunk & $\lambda_{\text{oracle}}$ \\
\midrule
C\_comp\_competitors\_s1 & $+0.0487$ & $-0.6571$ & $0.10$ \\
H\_heldout\_ho0p5\_s2 & $+0.1928$ & $-0.4605$ & $0.70$ \\
C\_comp\_competitors\_s2 & $+0.1569$ & $-0.4902$ & $0.70$ \\
H\_heldout\_ho0p7\_s2 & $+0.0667$ & $-0.5733$ & $0.70$ \\
H\_heldout\_ho0p9\_s2 & $+0.1579$ & $-0.4605$ & $0.70$ \\
\bottomrule
\end{tabular}
\caption{The five configurations where the two ceiling definitions differ most. A quantity
that can sit $0.65$ below the marginal threshold it is meant to bound is not a ceiling; the shrunk
version cannot be exceeded by PCC by construction, since PCC optimises the same $\lambda$
over a worse $\delta$.}
\label{tab:oracle}
\end{table}

\section{Score functions, miscoverage level and calibration depth}
\label{app:depth}

\Cref{tab:depth} sweeps the three axes \hyperref[sec:scores]{Sec.~4.5} and \hyperref[sec:depth]{Sec.~4.3}
report. The improvement survives every score and is strongest with LAC; it
attenuates for the adaptive ones, is zero at $\alpha=0.01$, and rises with
calibration depth over a range where the oracle ceiling is flat, which is what
makes the last block a statement about the method rather than about the ceiling.

\begin{table}[ht]
\centering
\small
\setlength{\tabcolsep}{3.5pt}
\begin{tabular}{llcc}
\toprule
Axis & Setting & $\Delta$ worst-class & \% ceiling \\
\midrule
\multirow{5}{*}{Score}
 & LAC \cite{sadinle2019least}                  & $\mathbf{+0.058}\ [+0.006,+0.110]$ & $42$ \\
 & PAS \cite{ding2026conformal}                 & $+0.066\ [-0.017,+0.149]$ & $28$ \\
 & SAPS \cite{huang2024conformal}               & $+0.031\ [-0.019,+0.081]$ & $24$ \\
 & APS \cite{romano2020classification}          & $+0.022\ [+0.004,+0.040]$ & $15$ \\
 & RAPS \cite{angelopoulos2021uncertainty}      & $+0.010\ [-0.009,+0.029]$ & $10$ \\
\midrule
\multirow{3}{*}{$\alpha$}
 & $0.10$ & $\mathbf{+0.058}\ [+0.006,+0.110]$ & $42$ \\
 & $0.05$ & $+0.040\ [-0.008,+0.088]$ & $39$ \\
 & $0.01$ & $+0.000\ [\ \ 0.000,\ 0.000]$ & $0$ \\
\midrule
\multirow{5}{*}{\shortstack[l]{Cal.\\depth}}
 & $10$/class  & $+0.033\ [+0.016,+0.050]$ & $26$ \\
 & $25$/class  & $+0.030\ [+0.011,+0.049]$ & $23$ \\
 & $50$/class  & $+0.041\ [+0.021,+0.061]$ & $32$ \\
 & $100$/class & $\mathbf{+0.060}\ [+0.032,+0.088]$ & $46$ \\
 & $175$/class & $+0.059\ [+0.031,+0.086]$ & $46$ \\
\bottomrule
\end{tabular}
\caption{Score function, miscoverage level and calibration depth on held-out ImageNet classes,
five splits, ten for calibration depth. The ceiling is flat across the depth block, $0.128$ to
$0.132$, so the rise there comes from the method and not from a wider ceiling.}
\label{tab:depth}
\end{table}

\section{Ablations}
\label{app:ablation}

\hyperref[tab:ablation]{Tab.~2} is the evidence for \hyperref[sec:ablations]{Sec.~4.2}. Each row removes one component
over ten ImageNet splits at $\alpha=0.10$, and a positive value means the full method is
better than the ablated one.

% ---------------------------------------------------------------- ablations

\section{Why recalibration is invisible at matched size}
\label{app:invisible}

Let $t_y=\hat q+\tilde\delta_y+c$ be the thresholds a variant emits and let $T$ be the target
average set size. The matched comparison reports each variant at $t_y+\sigma$, where $\sigma$
solves $\frac1N\sum_i|\{y:s(X_i,y)\le t_y+\sigma\}|=T$. Adding a constant $a$ to every
threshold replaces $t_y$ by $t_y+a$ and hence $\sigma$ by $\sigma-a$, leaving
$t_y+a+(\sigma-a)=t_y+\sigma$ unchanged, so the matched thresholds are identical. Marginal recalibration adds exactly such a
constant, and so does the $n^\star$ rule when it substitutes one global value, which is why both
ablations must read $0.0000$ in \hyperref[tab:ablation]{Tab.~2}. The zero is a property of the comparison and
carries no information about the component.

\Cref{tab:invisible} gives the view in which the component is visible, namely the thresholds as
emitted, over ten ImageNet splits. Recalibration buys $0.0026$ of marginal coverage for $0.31$
of average set size. Two further readings matter. First, the raw \textsc{Pcc} thresholds are far
wider than the marginal ones, $2.31$ against $1.21$, so this view cannot be used to compare
methods and exists only to show where marginal validity lives; the size matching in the main
tables removes exactly this difference. Second, \textsc{Pcc} lands at $0.8989$ against a target
of $0.900$ while the marginal threshold lands at $0.9013$ on the same rows. The deficit of
$0.0024$ is about thirty times the finite-sample slack $1/(n+1)$ and is the measured cost of
letting $g_\theta$, $\lambda$ and $c$ share calibration rows, as discussed in
\hyperref[sec:guarantee]{Sec.~3.5}.

\begin{table}[ht]
\centering
\small
\setlength{\tabcolsep}{4pt}
\begin{tabular}{lcc}
\toprule
Thresholds as emitted & Marginal cov. & Set size \\
\midrule
Marginal threshold $\hat q$        & $0.9013$ & $1.21$ \\
\textsc{Pcc}, no recalibration     & $0.8963$ & $2.00$ \\
\textsc{Pcc}, with recalibration   & $0.8989$ & $2.31$ \\
\midrule
Target                             & $0.9000$ & --- \\
\bottomrule
\end{tabular}
\caption{The unmatched view, ten ImageNet splits, $\alpha=0.10$. Recalibration moves marginal
coverage towards the target and costs set size. The matched comparison used everywhere else
cancels this shift by construction.}
\label{tab:invisible}
\end{table}

\section{Additional backbones, endogenous and exogenous descriptors}
\label{app:backbones}

We hypothesised that the two additional backbones failed because each was given its own
classifier head as $\varphi$, making the descriptors endogenous to the model that produced the
scores. \Cref{tab:backbones} refutes it: with a fixed ResNet-50 head instead --- whose weights
correlate $0.001$ element-wise with each model's own head, so it is genuinely exogenous ---
every interval still contains zero.

The reason is visible in the last column. All six arms have $35$ evaluation examples per class,
because ImageNet's validation split provides only $50$ per class in total and some must go to
calibration. \hyperref[fig:sweeps]{Fig.~2}a places $35$ below the depth at which any effect is
detectable, so these experiments are inconclusive rather than negative. Settling the question
needs an evaluation set larger than the validation split.

\begin{table}[ht]
\centering
\small
\setlength{\tabcolsep}{4pt}
\begin{tabular}{@{}llc@{}}
\toprule
Backbone & $\varphi$ & $\Delta$ worst-class \\
\midrule
ConvNeXt-T \cite{liu2022convnet}     & own   & $+0.000\ [-0.000,+0.000]$ \\
                                     & fixed & $+0.011\ [-0.017,+0.040]$ \\
ResNet-50 \cite{he2016deep}          & own   & $+0.017\ [-0.005,+0.040]$ \\
                                     & fixed & $+0.006\ [-0.005,+0.017]$ \\
ViT-B/16 \cite{dosovitskiy2021image} & own   & $-0.029\ [-0.072,+0.015]$ \\
                                     & fixed & $-0.011\ [-0.066,+0.043]$ \\
\bottomrule
\end{tabular}
\caption{Both descriptor arms on three backbones, five splits. ``own'' uses each model's own
classifier head, ``fixed'' the ResNet-50 head also used for ImageNet. Top-1 accuracies of all
three matched published values to four decimal places, which we used as a correctness check on
the weight extraction.}
\label{tab:backbones}
\end{table}

\section{Long-tailed datasets under a measurability threshold}
\label{app:longtail}

\paragraph{How many classes actually train $g_\theta$.} A class trains the regression only if
its own offset is observable, which needs $n_{\mathrm{cal}}=25$ calibration rows. That count,
not the size of the label space, is what separates the three datasets. ImageNet supplies $700$
of its $700$ labelled classes; Pl@ntNet-300K supplies $75$ of $757$ and iNaturalist-2018 $79$
of $5699$. Both long-tailed sets therefore fit nine coefficients from fewer than eighty
observations, an order of magnitude below ImageNet, \emph{despite} having more classes in
total: a long tail spends the dump on classes too thin to measure a $\delta_y$ for, so raising
$K$ lowers the number of usable regression targets rather than raising it. The cross-validated
errors of $0.036$, $0.072$ and $0.127$ track that count.

This also says what does \emph{not} separate Pl@ntNet from iNaturalist. Their fits are built
from $75$ and $79$ classes, the same number for this purpose, yet one yields a significant
improvement under the restriction below and the other yields nothing. The difference is
downstream of the regression: iNaturalist's oracle ceiling is $0.018$, so there is almost
nothing to take even with a perfect $\delta$, and its held-out space is $2443$ classes against
Pl@ntNet's $324$. $K$ is not the axis; measurable calibration depth, and the ceiling it leaves,
are.

\begin{table}[ht]
\centering
\small
\setlength{\tabcolsep}{3pt}
\begin{tabular}{@{}llcc@{}}
\toprule
Dataset & Thr. & Kept & $\Delta$ worst-class \\
\midrule
Pl@ntNet  & $\ge 35$ & $151$ & $+0.020\ [-0.011,+0.052]$ \\
Pl@ntNet  & $\ge 75$ & $98$  & $\mathbf{+0.019}\ [+0.010,+0.028]$ \\
iNat-2018 & $\ge 35$ & $237$ & $+0.000$ \\
iNat-2018 & $\ge 75$ & $75$  & $+0.000$ \\
\bottomrule
\end{tabular}
\caption{Restricting evaluation to classes whose slice can support a per-class estimate. Five
splits. Pl@ntNet recovers a significant improvement; iNaturalist does not, and there
$\lambda=0$ means the method declines to act rather than acting and failing.}
\label{tab:longtail}
\end{table}

The threshold changes the label space, so the row counts matter. On Pl@ntNet-300K the
$\ge 75$ rule keeps $98$ of $1081$ classes, of which $34$ are held out, and the survivors carry
a median of $116$ calibration and $198$ evaluation examples each. At $\ge 35$ it keeps $151$
classes with medians of $69$ and $116$. The improvement is significant at the stricter threshold
and not at the looser one, which is the ordering \hyperref[fig:sweeps]{Fig.~2}a predicts, since $35$ sits
at the edge of measurability while $75$ is clear of it.

On iNaturalist-2018 the survivors are not starved either: at $\ge 35$ the $237$ surviving
classes hold a median of $40$ calibration and $56$ evaluation examples. \Cref{tab:inatcurve}
gives the shrinkage curve on its labelled classes and explains the zero. The curve decreases
monotonically, so $\hat\delta$ is not degenerate; it carries signal, and the signal is wrong.
Because the grid contains $\lambda=0$, the selection step returns the marginal threshold and the
method neither helps nor harms. The cross-validated error of $g_\theta$ is $0.127$ here against
$0.036$ on ImageNet, so the regression is roughly three and a half times worse in a label space
eight times larger, fitted in a descriptor space of the same dimension. A richer regression
head does not recover it: the quadratic and kernel heads of \cref{app:heads} abstain in all
five splits here, exactly as the ridge does. Whether richer descriptors would is untested, and
is the open question.

\begin{table}[ht]
\centering
\small
\setlength{\tabcolsep}{4pt}
\begin{tabular}{@{}lccccc@{}}
\toprule
$\lambda$ & $0$ & $0.1$ & $0.3$ & $0.5$ & $1.0$ \\
\midrule
worst & $\mathbf{0.699}$ & $0.466$ & $0.385$ & $0.350$ & $0.319$ \\
\bottomrule
\end{tabular}
\caption{Shrinkage curve on the labelled classes of iNaturalist-2018, mean over five splits, at
the $\ge 75$ threshold. Every positive $\lambda$ is worse than none, so the selection step
returns $\lambda=0$. A flat curve would have meant $\hat\delta$ was degenerate; a decreasing one
means it is informative and points the wrong way.}
\label{tab:inatcurve}
\end{table}

\section{ImageNet-C, per corruption}
\label{app:imagenetc}

\Cref{tab:imagenetc} reports all $15$ corruptions at severities $3$ and $5$, three splits
each. Calibration uses clean images and evaluation uses corrupted ones, so exchangeability
fails and no method retains a coverage guarantee here, as the corrupted-input
experiment of the main paper states.

The empty-set column is the direct evidence of that. It ranges from
$0.063$ (brightness, severity 3) to
$0.757$ (contrast, severity 5), and an
average set size below $1.0$ is only possible when a large share of test points receive
\emph{nothing}. A threshold calibrated on clean images does not give narrow sets on corrupted
ones; it abstains. Coverage and set size both merely look ``low'' in that regime, which is why
we added this column.

\begin{table}[ht]
\centering
\small
\setlength{\tabcolsep}{3pt}
\begin{tabular}{llcccc}
\toprule
Corruption & Sev. & $\Delta$ & Cov. & Size & Empty \\
\midrule
brightness & 3 & $+0.0200$ & $0.841$ & $1.345$ & $0.063$ \\
brightness & 5 & $+0.0067$ & $0.746$ & $1.362$ & $0.103$ \\
contrast & 3 & $-0.0133$ & $0.600$ & $1.297$ & $0.171$ \\
contrast & 5 & $+0.0000$ & $0.053$ & $0.317$ & $0.757$ \\
defocus\_blur & 3 & $+0.0067$ & $0.500$ & $1.146$ & $0.268$ \\
defocus\_blur & 5 & $+0.0000$ & $0.229$ & $0.763$ & $0.497$ \\
elastic\_transform & 3 & $+0.0000$ & $0.706$ & $1.355$ & $0.112$ \\
elastic\_transform & 5 & $+0.0000$ & $0.264$ & $1.422$ & $0.164$ \\
fog & 3 & $+0.0133$ & $0.629$ & $1.345$ & $0.143$ \\
fog & 5 & $+0.0000$ & $0.360$ & $1.157$ & $0.264$ \\
frost & 3 & $+0.0000$ & $0.441$ & $1.187$ & $0.247$ \\
frost & 5 & $+0.0000$ & $0.331$ & $1.095$ & $0.305$ \\
gaussian\_noise & 3 & $+0.0067$ & $0.509$ & $1.266$ & $0.201$ \\
gaussian\_noise & 5 & $+0.0000$ & $0.096$ & $0.636$ & $0.545$ \\
glass\_blur & 3 & $+0.0000$ & $0.278$ & $1.076$ & $0.326$ \\
glass\_blur & 5 & $+0.0000$ & $0.141$ & $0.762$ & $0.500$ \\
impulse\_noise & 3 & $+0.0133$ & $0.502$ & $1.309$ & $0.188$ \\
impulse\_noise & 5 & $+0.0000$ & $0.113$ & $0.674$ & $0.523$ \\
jpeg\_compression & 3 & $+0.0200$ & $0.772$ & $1.372$ & $0.085$ \\
jpeg\_compression & 5 & $+0.0067$ & $0.522$ & $1.309$ & $0.184$ \\
motion\_blur & 3 & $+0.0067$ & $0.512$ & $1.206$ & $0.236$ \\
motion\_blur & 5 & $+0.0000$ & $0.220$ & $0.885$ & $0.425$ \\
pixelate & 3 & $+0.0067$ & $0.685$ & $1.378$ & $0.112$ \\
pixelate & 5 & $+0.0000$ & $0.371$ & $1.190$ & $0.258$ \\
shot\_noise & 3 & $+0.0000$ & $0.470$ & $1.244$ & $0.213$ \\
shot\_noise & 5 & $+0.0000$ & $0.113$ & $0.714$ & $0.492$ \\
snow & 3 & $+0.0000$ & $0.494$ & $1.297$ & $0.184$ \\
snow & 5 & $+0.0000$ & $0.262$ & $1.002$ & $0.349$ \\
zoom\_blur & 3 & $+0.0000$ & $0.528$ & $1.280$ & $0.190$ \\
zoom\_blur & 5 & $+0.0000$ & $0.367$ & $1.152$ & $0.272$ \\
\bottomrule
\end{tabular}
\caption{ImageNet-C, per corruption and severity, three splits. $\Delta$ worst is the change
in worst-class coverage on held-out classes at matched set size. Pooled over all $90$
conditions the improvement is $+0.003$ $[+0.001,+0.005]$.}
\label{tab:imagenetc}
\end{table}

\section{What the regression fits}
\label{app:fit}

\begin{figure}[h]
\centering
\includegraphics[width=\linewidth]{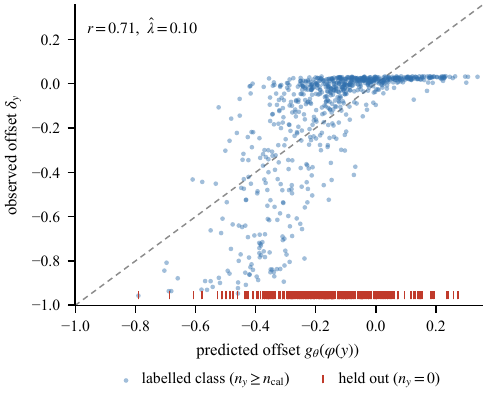}
\caption{The mechanism, on the primary ImageNet configuration. Every labelled class
contributes an observed offset $\delta_y$ and a predicted one $g_\theta(\varphi(y))$; held-out
classes appear as ticks on the horizontal axis, because the observed coordinate does not exist
for them. The regression recovers the direction of the correction and understates its
magnitude, and the classes that decide the minimum lie in the dense band near $\delta_y=0$,
where the cost is prediction error rather than shrinkage towards the mean. That error is the
quantity \hyperref[prop:sup]{Prop.~1} bounds, though \cref{app:lipschitz} finds that bound is not what selects $\hat\lambda$.}
\label{fig:method}
\end{figure}

\paragraph{Which classes enter the fit.} A labelled class contributes an observed offset only
if it holds at least $n_{\mathrm{cal}}=25$ calibration rows. On the primary ImageNet
configuration that threshold never binds: the thinnest labelled class holds $101$ rows and the
median $177$, so all $700$ enter. It binds on the long-tailed sets, where the median labelled
class holds two, and there the exclusion is the measurability problem of \cref{app:longtail}
rather than a separate design choice. Weighting thin classes by their row count instead of
excluding them would therefore change nothing on ImageNet, and on the long-tailed sets would
admit a per-class quantile estimated from two observations.

\section{Which classes the correction moves}
\label{app:whohelped}

Worst-class coverage is a minimum, and a minimum says nothing about the shape it was taken
from. \Cref{fig:coverage} therefore plots the whole sorted curve and \cref{tab:whohelped}
reports the mean change by prevalence quintile, with the fraction of classes that rise and
fall. Both come from the first split of the primary configuration, the only one for which we
stored the per-class breakdown.

\begin{figure*}[t]
\centering
\includegraphics[width=\linewidth]{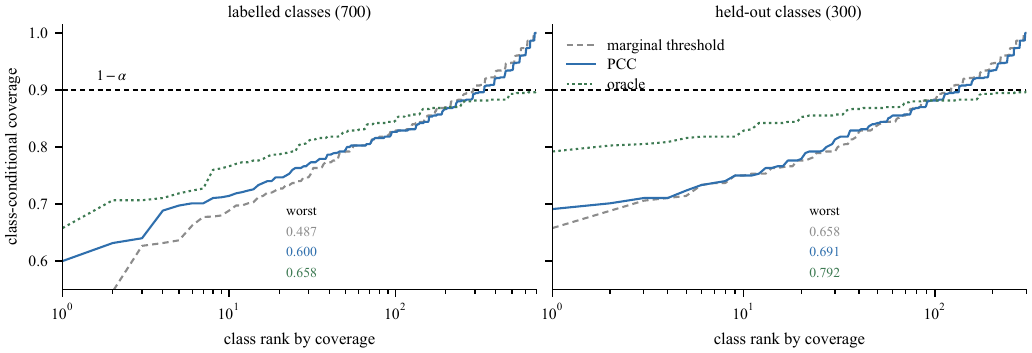}
\caption{Class-conditional coverage of every class, sorted, on the primary ImageNet
configuration. Rank rather than percentile on the horizontal axis, because worst-class
coverage is decided by the leftmost few classes and a linear axis compresses them to nothing.
The correction rotates the curve: the tail rises and the classes that had coverage to spare
give some back, with the crossing near rank $150$ of $300$ on the held-out classes. The oracle
performs the same rotation further, which is the sense in which it is a ceiling rather than a
different kind of method.}
\label{fig:coverage}
\end{figure*}

The curve rotates rather than shifting. On the held-out classes the lowest-ranked class rises
from $0.658$ to $0.691$ while the median falls from $0.921$ to $0.909$, and macro-coverage moves
from $0.904$ to $0.900$. The oracle rotates the same curve further, lifting its worst class to
$0.792$ and pulling its median to $0.883$; what separates the two is how far the rotation goes,
not which way. Because the comparison holds average set size fixed, coverage given to the
classes that were failing has to come from somewhere, and it comes in thin slices from the
classes that had it to spare. Whether that trade is worth making is a decision about the
application rather than about the method, which is why we report both the minimum and the
average throughout.

Two cautions on reading the fractions in \cref{tab:whohelped}. At the $76$ evaluation examples
per class of this configuration a single example is $0.013$ of coverage, so how many classes
register as changed at all depends on evaluation depth as much as on the correction; the mean
change of $-0.004$ is the quantity that survives a change of depth. And the worst class is not
the same class before and after, class $754$ before the correction and class $663$ after, so the
two ends of that arrow are not one class being repaired.

Prevalence carries little information here, because ImageNet's held-out classes hold between
$693$ and $1199$ rows in the score dump and are therefore close to balanced. The quintiles are
included so that the same table can be read for the long-tailed datasets, where prevalence
varies over three orders of magnitude.

\begin{table}[ht]
\centering
\small
\setlength{\tabcolsep}{4pt}
\begin{tabular}{@{}lccc@{}}
\toprule
Quintile & Classes & Med.\ prev. & Mean $\Delta\hat c_y$ \\
\midrule
1, frequent & $60$ & $1182$ & $-0.0078$ \\
2 & $60$ & $1174$ & $-0.0033$ \\
3 & $60$ & $1169$ & $-0.0033$ \\
4 & $60$ & $1163$ & $-0.0039$ \\
5, rare & $60$ & $1147$ & $-0.0024$ \\
\midrule
\multicolumn{4}{@{}l@{}}{Improved $11.3\%$, worsened $43.3\%$} \\
\multicolumn{4}{@{}l@{}}{Worst class $0.658 \rightarrow 0.691$} \\
\bottomrule
\end{tabular}
\caption{Change in per-class coverage on the $300$ held-out ImageNet classes, by prevalence
quintile, at matched average set size. The correction concentrates its effect on the failing
classes and pays for it with small reductions spread across the rest.}
\label{tab:whohelped}
\end{table}

\section{Score-level interventions do not equalise class coverage}
\label{app:score-level}

\hyperref[sec:setup]{Sec.~3.1} argues that improving $s(x,y)$ cannot equalise coverage across classes,
because any such change moves every class in the same direction and supplies no per-class
degree of freedom. \Cref{tab:scorelevel} is the measurement behind that argument. We built four
score functions from the same backbone, the plain softmax and three aggregations over augmented
copies of the input \cite{shanmugam2021better}, and evaluated each with the marginal threshold across two miscoverage
levels, three descriptor families and three splits, for $48$ configurations.

The four scores are close to indistinguishable in worst-class coverage. At $\alpha=0.10$ and
matched to within $0.03$ in average set size they give $0.534$, $0.485$, $0.534$ and $0.518$,
so the spread across four quite different score constructions is $0.049$, and two of the three
augmentation-based scores are \emph{worse} than the plain softmax. Their macro-coverage and
marginal coverage agree to three decimal places. This is what a score-level intervention looks
like when the objective is a minimum over classes.

Two honest caveats. The comparison is on one backbone and one dataset, so it supports the
structural argument and does not substitute for it. And this configuration has a median of $25$
evaluation examples per class, which \hyperref[sec:depth]{Sec.~4.3} places below the measurability threshold,
so the near-zero effect of \textsc{Pcc} in the last column of \cref{tab:scorelevel} is
uninformative about \textsc{Pcc}. That configuration is in fact the fifth instance of the
evaluation-depth pattern we found, and we found it before we understood the pattern, which is
part of why the diagnosis in \hyperref[sec:depth]{Sec.~4.3} took three attempts.

\begin{table}[ht]
\centering
\small
\setlength{\tabcolsep}{3pt}
\begin{tabular}{lcccc}
\toprule
Score & Worst-class & Set size & Macro & $\Delta$ \textsc{Pcc} \\
\midrule
Softmax (LAC)          & $0.534$ & $1.256$ & $0.900$ & $-0.005$ \\
Average over augment.  & $0.485$ & $1.228$ & $0.898$ & $+0.002$ \\
Learned aggregation    & $0.534$ & $1.233$ & $0.899$ & $+0.001$ \\
Uncertainty-modulated  & $0.518$ & $1.230$ & $0.900$ & $-0.003$ \\
\bottomrule
\end{tabular}
\caption{Four score functions on the same backbone under the marginal threshold, held-out
classes, $\alpha=0.10$, three splits. Changing the score leaves worst-class coverage within
$0.049$ and does not order the scores the way marginal coverage would. The last column reports
what \textsc{Pcc} adds on top of each, in a configuration whose $25$ evaluation examples per
class put it below the depth at which the statistic can move (\hyperref[sec:depth]{Sec.~4.3}).}
\label{tab:scorelevel}
\end{table}

\section{Splitting the calibration rows}
\label{app:split}

\hyperref[prop:marginal]{Prop.~2} requires the rows used to build $\tilde\delta$ and the rows used to
calibrate $c$ to be disjoint, and the configuration reported throughout \hyperref[sec:exp]{Sec.~4} does not
provide that: $g_\theta$, $\lambda$, $n_{\mathrm{cal}}$ and $c$ all read the same calibration
rows. We measured the cost of the reuse in \cref{app:invisible}, where \textsc{Pcc} lands
$0.0024$ of marginal coverage below the marginal threshold on the same rows.

\begin{table}[ht]
\centering
\small
\setlength{\tabcolsep}{3pt}
\begin{tabular}{@{}lll@{}}
\toprule
Quantity & Rows & Classes \\
\midrule
$\hat q$ (1)      & calib. & labelled \\
$\delta_y$ (3)      & calib. & labelled, $n_y\ge n_{\mathrm{cal}}$ \\
$g_\theta$                       & calib. & labelled, $n_y\ge n_{\mathrm{cal}}$ \\
$\lambda$ (5)      & calib. & labelled \\
$n_{\mathrm{cal}}$               & calib. & labelled \\
$c$ (9)               & calib. & labelled \\
\midrule
\emph{all reported numbers}      & \emph{eval.} & \emph{both tables} \\
\bottomrule
\end{tabular}
\caption{What each estimated quantity may see; the equation numbers are those of the main
paper. Evaluation rows and held-out classes are used only for reporting. The default configuration lets $g_\theta$, $\lambda$ and $c$ share the same
calibration rows, which is convenient but not what \hyperref[prop:marginal]{Prop.~2} assumes; \textsc{Pcc}-split
divides those rows in two and is the variant the proposition covers.}
\label{tab:splits}
\end{table}

\textsc{Pcc}-split is the variant the proposition covers. The calibration rows of each labelled
class are divided by a per-class random split; the first part estimates $\delta_y$, fits
$g_\theta$ and selects $\lambda$ and $n_{\mathrm{cal}}$, and the second part is used for nothing
except the conformal quantile that gives $c$. Held-out classes and evaluation rows are untouched
in both variants, so the comparison isolates the effect of the reuse.

The cost is statistical and falls entirely on $g_\theta$. The offset $c$ is a single scalar
computed from pooled rows, so it needs a large total and not depth within each class: at the
two fractions we report it is estimated from roughly $31$k and $62$k rows, whose conformal
slack is $3\times10^{-5}$ and $2\times10^{-5}$, both two orders of magnitude below the
$0.0024$ deficit being measured. The regression is the part that pays, since it needs depth
per class: the reuse variant fits on $174$ calibration examples per class and the two split
variants on $131$ and $87$. \hyperref[tab:ablation]{Tab.~2} puts the price of that at roughly $0.019$ per
halving. We report two fractions rather than one so that the choice cannot be read as tuned.

\Cref{tab:split} reports both arms against the reuse variant over ten paired splits.
Satisfying the premise removes the deficit and costs nothing measurable. At a quarter of the
calibration rows held back, marginal coverage moves from $0.8994$ to $0.9011$ against a target
of $0.900$, and the deficit relative to the marginal threshold on the same rows falls from
$0.0018$ to zero at four decimal places. Paired split by split the improvement in marginal
coverage is $+0.00161$ with interval $[+0.00084,+0.00239]$, which excludes zero, while the
change in worst-class coverage is $+0.0048$ with interval $[-0.0046,+0.0142]$, which does not.

We expected the split to cost something, since $g_\theta$ loses a quarter of its calibration
depth, and within the resolution of ten splits it does not. The likely reason is visible in
\hyperref[tab:ablation]{Tab.~2}: worst-class improvement is flat between $100$ and $175$ calibration examples
per class, and the split moves us from $174$ to $131$, which is inside that flat region. At
half the rows held back, leaving $87$ per class, the point estimate does fall back to the
reuse level, which is the direction the depth curve predicts.

The recommendation that follows is to use the split variant. It is what \hyperref[prop:marginal]{Prop.~2}
covers, its coverage lands on target rather than below it, and the price at a quarter is not
measurable here. \hyperref[tab:competitors]{Tab.~1} is reported under the split variant for that reason. The
remaining tables use reuse, which is what every earlier run used, and the difference is now
measured rather than argued: over five paired splits at the protocol of \hyperref[sec:exp]{Sec.~4}, with
the published baselines fitted alongside, the two variants give worst-class improvements that
agree to four decimals, a paired difference of $+0.0000$ $[+0.0000,+0.0000]$. The lower-tail
metrics do separate them, by $0.0006$ at the fifth percentile, which is the resolution
worst-class coverage lacks.

\begin{table}[ht]
\centering
\small
\setlength{\tabcolsep}{3pt}
\begin{tabular}{@{}lccc@{}}
\toprule
Variant & Marg.\ cov. & Deficit & $\Delta$ worst \\
\midrule
Reuse         & $0.8994$ & $-0.0018$ & $+0.0588$ \\
Split $25\%$  & $\mathbf{0.9011}$ & $\mathbf{-0.0000}$ & $+0.0636$ \\
Split $50\%$  & $0.9003$ & $-0.0007$ & $+0.0585$ \\
\midrule
Target        & $0.9000$ & ---       & --- \\
\bottomrule
\end{tabular}
\caption{\textsc{Pcc}-split against reuse, ten paired splits on ImageNet at $\alpha=0.10$.
``Deficit'' is marginal coverage minus that of the marginal threshold measured on the same
evaluation rows, so it isolates the optimism from reusing calibration rows rather than the
finite-sample slack common to both. $\Delta$ worst-class is the matched-size statistic used
throughout the paper.}
\label{tab:split}
\end{table}

\section{The row budget, and whether the results depend on it}
\label{app:rowbudget}

Every table in this paper is computed on a random subsample of $250{,}000$ rows from a score
dump of $1{,}153{,}051$, and the reason is worth stating plainly because it began as a
practical one. The full dump is a $1{,}153{,}051\times1000$ matrix of \texttt{float32}, which
is $4.61$\,GB, and holding it alongside the score transform and the calibration and evaluation
slices exceeds the memory of the machines we ran on.

Three properties make the subsample a valid operation rather than a convenience. It is drawn
as a fraction \emph{within each class} rather than as a global cap, which preserves the
prevalence structure of the dump; a global cap would equalise the class counts and make the
prevalence feature constant, which would in turn make the prevalence ablation vacuous. Because
the draw is uniform at random and does not depend on any score or label value, the calibration
and test rows remain exchangeable, so \hyperref[prop:marginal]{Prop.~2} is unaffected. And the depth that
remains, about $250$ rows per class, exceeds the $84$ per class that the pre-registered premise
required, by a factor of $2.7$.

What the subsample does cost is precision, and the direction of that cost is the part a reader
should be able to check rather than take on trust. Both of the depth results in this paper say
the cost runs against us: \hyperref[tab:ablation]{Tab.~2} shows worst-class improvement rising with
calibration depth, from $0.033$ at ten examples per class to $0.059$ at $175$, and
\hyperref[fig:sweeps]{Fig.~2}a shows the measurable effect rising with evaluation depth. A larger budget
should therefore give a \emph{larger} number, which would make $250{,}000$ conservative.

\Cref{tab:rowbudget} sweeps the budget over a factor of three, at five splits each. The
prediction above is not what we found, and neither is its opposite. The improvement does not
vary detectably with the budget: the three point estimates are $0.058$, $0.050$ and $0.076$,
each of them lies inside the confidence intervals of the other two, and the regression of the
improvement on $\log$ rows gives a slope of $+0.013$ per $e$-fold, which five splits cannot
distinguish from zero. The middle budget is the lowest of the three, so the sequence is not
even monotone.

For the question this section exists to answer, that is a better result than the one we
predicted. The row budget was chosen for a practical reason, and the measurement says the
choice does not determine the conclusion over the range we can afford to test. The paper's
headline is not an artefact of the subsample.

One trend in the table is real and is a consequence of the paper's own argument. The oracle
ceiling falls as the budget grows, from $0.140$ to about $0.116$. Worst-class coverage is a
minimum over three hundred classes, and a minimum computed from $76$ evaluation examples per
class is more extreme than one computed from $228$, for the method and for the oracle alike.
Deeper evaluation makes the statistic less noisy and therefore less generous, which is the
same effect \hyperref[fig:sweeps]{Fig.~2}a measures from the other direction.

\begin{table}[ht]
\centering
\small
\setlength{\tabcolsep}{3pt}
\begin{tabular}{@{}rcclc@{}}
\toprule
Rows & Cal. & Eval. & $\Delta$ worst-class & Ceil. \\
\midrule
$250$k & $177$ & $\ 76$ & $+0.058\ [+0.006,+0.110]$ & $0.140$ \\
$500$k & $354$ & $152$ & $+0.050\ [+0.011,+0.089]$ & $0.105$ \\
$750$k & $532$ & $228$ & $+0.076\ [+0.034,+0.118]$ & $0.116$ \\
\bottomrule
\end{tabular}
\caption{The row budget swept over a factor of three, five splits each, everything else at the
primary configuration. No point estimate falls outside the intervals of the other two. The
$250$k row reproduces the corresponding entry of \hyperref[fig:sweeps]{Fig.~2}b to four decimal places,
which is the check that the sweep shares the configuration of the rest of the paper.}
\label{tab:rowbudget}
\end{table}

\section{Descriptor choices}
\label{app:descriptors}

The nine features are $\|w_y\|_2$, $b_y$, the cosine distance from $w_y$ to its nearest
competing row, the mean cosine distance to its $k$ nearest rows for $k\in\{5,10,50\}$, the
spread of those distances, the cosine distance to the mean row, and $\log$ prevalence. A tenth
candidate, the cosine distance to the single nearest row computed at $k=1$, is the feature
closest in construction to the regression target and is held out of the feature set for that
reason. All nine are standardised to zero mean and unit variance before the ridge, using
the labelled classes' statistics, so nothing depends on their relative scales --- $\|w_y\|_2$
and $\log$ prevalence differ by an order of magnitude.

\paragraph{Heads that are not linear.} A model scoring by embedding similarity has no weight
row to read, but it has the object $w_y$ stands in for: a class prototype, whether a normalised
centroid or a learned embedding the logit takes an inner product with. Six of the nine features
are angular --- the nearest-row distance, the three $k$-neighbourhood means, their spread, and
the distance to the mean row --- and are defined on prototypes unchanged. The two that are not
fare worse. A cosine classifier normalises its rows, so $\|w_y\|_2$ is constant by
construction, and most such heads carry no bias, which would leave $\varphi$ with seven entries
and neither of its two scale features. That is arithmetic, not a result: whether the remaining
geometry carries enough signal is an empirical question, and we have not run it.

\hyperref[tab:ablation]{Tab.~2} ablates by feature group. Distances alone and prevalence alone are both
significantly worse than the full set, while removing prevalence is not, so the geometric
features carry the signal and the one label-derived feature is dispensable. The neighbour
counts $k\in\{5,10,50\}$ were fixed in advance and never tuned; $k=1$ enters only through the
nearest-row distance and the margin. The cross-validated error of $g_\theta$ is stable across
splits, at $0.0356$ to $0.0372$ over the five ImageNet splits, and the selected $\lambda$ is
$0.1$ in all five, which is the sense in which the fit is not delicate.

\Cref{tab:descriptors} varies the three choices a reader is most likely to question, at five
splits each. The metric matters and the other two do not. Replacing cosine distances with
Euclidean ones costs $0.013$ of worst-class improvement and raises the cross-validated error of
$g_\theta$ from $0.036$ to $0.040$, which is the only arm where the regression measurably
degrades. Dropping the head bias and changing the neighbour counts leave both quantities where
they were.

Two of those arms return numbers identical to the default at four decimal places, and that is
worth explaining rather than presenting as three independent confirmations. Worst-class
coverage is a minimum over three hundred classes estimated from about $76$ evaluation examples
each, so it is a coarse statistic dominated by one badly covered class. The feature sets do
differ and so do the fitted regressions, but neither change moves that particular class into or
out of a different coverage bucket. The honest reading is that the result is insensitive to
these two choices, not that they have no effect on $g_\theta$.

\begin{table}[ht]
\centering
\small
\setlength{\tabcolsep}{3pt}
\begin{tabular}{@{}llc@{}}
\toprule
Descriptors & $\Delta$ worst-class & CV err. \\
\midrule
Cosine, all       & $+0.058\ [+0.006,+0.110]$ & $0.036$ \\
Euclidean         & $+0.045\ [+0.009,+0.081]$ & $0.040$ \\
No head bias      & $+0.058\ [+0.006,+0.110]$ & $0.036$ \\
$k\in\{1,2,5\}$   & $+0.058\ [+0.006,+0.110]$ & $0.036$ \\
$k\in\{1,10,20,100\}$ & $+0.063\ [+0.016,+0.111]$ & $0.036$ \\
\bottomrule
\end{tabular}
\caption{Descriptor choices, five splits each on held-out ImageNet classes. ``CV error'' is the
cross-validated mean squared error of $g_\theta$ on the labelled classes. The neighbour count
held out as the distance reference follows the smallest $k$ requested, so each arm withholds its
own feature and none is advantaged.}
\label{tab:descriptors}
\end{table}

\section{The regression head}
\label{app:heads}

\hyperref[sec:correction]{Sec.~3.3} fits $g_\theta$ by ridge regression. Whether that linear form is what
limits the method is a fair question, and it is sharpest on iNaturalist-2018, where the
regression carries no usable signal at all. We replaced the ridge with two closed-form
nonlinear heads and changed nothing else: a \emph{quadratic} head, a ridge over the squares
and pairwise products as well as the features, and \emph{kernel} ridge with an RBF kernel.
Both are shrunk, selected and recalibrated by the same procedure as the linear head, so the
rows of \cref{tab:heads} are comparable with the rest of the paper. The class floor is checked
against the coefficients each head actually fits, which for the quadratic head over $p=9$
features is $54$ rather than $9$.

The linear head is not the constraint. On ImageNet the quadratic head moves the point estimate
by $0.0004$ and widens the interval to include zero, beating the linear head on one of the five
splits, tying on two and losing on two; its standard deviation across splits is $0.079$ against
the linear head's $0.059$. The kernel head is worse at every width we tried, and at
$\gamma=3$ it is worse than applying no correction.

On the long-tailed datasets the reason is visible rather than inferred. Pl@ntNet-300K is where
the extra capacity ought to pay, since there is a real effect to find, and there the shrinkage
step declines to use the nonlinear predictions: $\hat\lambda=0$ in three of five splits for
both nonlinear heads and in none of them for the linear head. On iNaturalist-2018 every head
abstains in every split, so what fails at $K=8142$ is not the functional form of $g_\theta$.
The oracle ceiling there is $0.018$, and that is the more immediate obstacle.

The direction is the one \hyperref[prop:sup]{Prop.~1} describes. Worst-class coverage is controlled by the
largest error in the predicted offsets, and a head with more capacity fits the labelled classes
more closely without any corresponding control on its worst error over classes it has never
seen. \Cref{app:lipschitz} measures that sup-norm error and finds it does not select $\hat\lambda$,
so this remains an interpretation rather than a second measurement; what is measured is that the extra capacity buys no point estimate, costs
interval width on every dataset, and is declined outright wherever the shrinkage step is free
to decline it.

\begin{table}[ht]
\centering
\small
\setlength{\tabcolsep}{3pt}
\begin{tabular}{@{}llccc@{}}
\toprule
Data & Head & $\Delta$ worst-class & $\bar\lambda$ & $\hat\lambda{=}0$ \\
\midrule
\multirow{5}{*}{\shortstack[l]{Image\\Net}}
 & linear    & $\mathbf{+0.058}\ [+0.006,+0.110]$ & $0.10$ & $0/5$ \\
 & quadratic & $+0.058\ [-0.011,+0.128]$ & $0.20$ & $0/5$ \\
 & kernel    & $+0.026\ [+0.008,+0.045]$ & $0.32$ & $0/5$ \\
 & \ \ $\gamma{=}0.3$ & $+0.046\ [+0.004,+0.088]$ & $0.23$ & $0/3$ \\
 & \ \ $\gamma{=}3$   & $-0.022\ [-0.065,+0.021]$ & $0.27$ & $0/3$ \\
\midrule
\multirow{3}{*}{\shortstack[l]{Pl@nt\\Net}}
 & linear    & $\mathbf{+0.019}\ [+0.010,+0.028]$ & $0.22$ & $0/5$ \\
 & quadratic & $+0.014\ [-0.009,+0.038]$ & $0.06$ & $3/5$ \\
 & kernel    & $+0.009\ [-0.005,+0.023]$ & $0.04$ & $3/5$ \\
\midrule
iNat & all & $+0.000\ [0.000,0.000]$ & $0.00$ & $5/5$ \\
\bottomrule
\end{tabular}
\caption{Replacing the ridge $g_\theta$ with a nonlinear head, everything else at the
configuration of the corresponding table. Five splits, except the two indented RBF-width rows,
which are three and carry their own ceiling of $0.153$. Unindented kernel rows use $\gamma=1$.
Oracle ceilings are $0.140$ on ImageNet, $0.086$ on Pl@ntNet-300K under the $75$-example
restriction of \cref{tab:longtail}, and $0.018$ on iNaturalist-2018. The last column counts
splits in which the shrinkage step declined the prediction entirely.}
\label{tab:heads}
\end{table}

\section{How much the head row has to be worth}
\label{app:noise}

Our zero-calibration classes are ImageNet classes whose calibration rows we delete, so the head
row supplying $\varphi(y)$ was fitted with their training data. A class that is genuinely rare
has a row fitted from few examples. Two questions follow: how much does the correction depend
on that row, and where do real rare classes sit relative to the classes we tested on.

The second is measurable without retraining, and without the confound that makes the direct
test unreadable --- in a long-tailed dataset a class with few training examples also has few
calibration rows, so its observed offset is noisy too, and a regression of prediction error on
class abundance cannot separate a noisy descriptor from a noisy target. Head geometry alone
carries no such confound: the head is fitted on the training split while abundance is read off
the calibration split, which the head never saw. \Cref{tab:headrow} reports it. In models
trained on genuinely long-tailed data, the rows of data-poor classes sit \emph{further} from
the mean row than the head classes', not closer, and on Pl@ntNet-300K the relationship is
strong. ImageNet is the control and shows nothing, as it must: its classes are balanced by
construction, so there is no variation for the measurement to track.

That direction matters, because the obvious perturbation is the wrong one. Shrinking a row
towards the mean models a class whose geometry has washed out; the measurement says real rare
rows move the other way. We therefore perturb outward, adding isotropic noise scaled to
$\|\varepsilon\| = \sigma\|w_y-\bar w\|$ for held-out classes only. In high dimension
$\varepsilon$ is nearly orthogonal to $w_y-\bar w$, so the distance-to-mean ratio becomes
$\sqrt{1+\sigma^2}$ and a measured ratio $r$ reads back as $\sigma=\sqrt{r^2-1}$: Pl@ntNet's
$1.611$ is $\sigma=1.26$ and iNaturalist's $1.072$ is $\sigma=0.39$. The realised ratios matched
the intended ones to three decimals in every run.

Held-out descriptors are read off the perturbed head while the labelled ones stay on the
original, so $\hat\lambda$ is selected exactly as in \hyperref[sec:exp]{Sec.~4} and the method cannot adapt
to the handicap. Two rows of \cref{tab:noise} are checks rather than results: $\sigma=0$
reproduces $+0.058$, and $\hat\lambda=0.10$ at every $\sigma$, confirming the perturbation never
reached the labelled side. An earlier version of this experiment shrank the head in one pass
and $\hat\lambda$ moved, because $\varphi$ is built from distances to competing rows and the
held-out rows are among the candidates; that version measured a method partly adapting to its
own handicap.

The correction rests on the geometry. It survives $\sigma=0.39$ at $+0.049$, is indistinguishable
from zero by $\sigma=0.60$, and by $\sigma=2$ is significantly \emph{harmful} at $-0.069$ with an
interval excluding zero. Shrinkage cannot rescue this, because $\hat\lambda$ is chosen on the
labelled classes whose rows are intact. The lower-tail metrics fall monotonically alongside, so
what degrades is the tail as a whole and not one class.

What this does \emph{not} explain is the long-tailed failures, and we had expected it to. The
two datasets sit on opposite sides of the sweep's breakpoint from their outcomes: iNaturalist
has the milder perturbation, $\sigma=0.39$, where the method still works, yet it abstains
entirely; Pl@ntNet has the severe one, $\sigma=1.26$, where the sweep says the effect is gone,
yet it is where we measure a significant improvement. Head-row geometry is a real sensitivity
and not the mechanism behind those results, which \cref{app:longtail} continues to attribute to
regression quality and to a near-zero oracle ceiling. The sweep transfers a magnitude measured
on one dataset onto another, and the inversion is the evidence that such a transfer does not
carry the explanation with it.

\begin{table}[ht]
\centering
\small
\setlength{\tabcolsep}{3pt}
\begin{tabular}{@{}lccc@{}}
\toprule
Data & $r$ & $\sigma=\sqrt{r^2-1}$ & $\rho$ with abundance \\
\midrule
ImageNet, control & $1.04$ & $0.28$ & $-0.09$ \\
iNaturalist-2018  & $1.07$ & $0.39$ & $-0.36$ \\
Pl@ntNet-300K     & $\mathbf{1.61}$ & $\mathbf{1.26}$ & $\mathbf{-0.88}$ \\
\bottomrule
\end{tabular}
\caption{Head-row geometry against class abundance. $r$ is the ratio of median
$\|w_y-\bar w\|$ between the bottom and top deciles of calibration depth, and $\rho$ is the
Spearman correlation of $\|w_y\|$ with log abundance. ImageNet is the control: its deciles
differ by a factor of $1.0$ in abundance, so it has almost no variation to track, and it shows
none.}
\label{tab:headrow}
\end{table}

\begin{table}[ht]
\centering
\small
\setlength{\tabcolsep}{3pt}
\begin{tabular}{@{}lccc@{}}
\toprule
$\sigma$ & ratio & $\Delta$ worst-class & CVaR$_{5\%}$ \\
\midrule
$0.00$ & $1.00$ & $\mathbf{+0.058}\ [+0.006,+0.110]$ & $0.732$ \\
$0.39$ & $1.07$ & $+0.049\ [+0.000,+0.098]$ & $0.725$ \\
$0.60$ & $1.17$ & $+0.043\ [-0.014,+0.099]$ & $0.718$ \\
$1.26$ & $1.61$ & $-0.014\ [-0.072,+0.043]$ & $0.672$ \\
$2.00$ & $2.24$ & $-0.069\ [-0.127,-0.012]$ & $0.625$ \\
\bottomrule
\end{tabular}
\caption{Held-out head rows perturbed outward by $\sigma$, five splits each, everything else at
the primary configuration. $\sigma=0.39$ and $\sigma=1.26$ are the magnitudes
\cref{tab:headrow} measures for iNaturalist-2018 and Pl@ntNet-300K. The oracle ceiling is
$0.140$ throughout, since the oracle reads the test labels and not $\varphi$. CVaR$_{5\%}$ is
the mean coverage of the worst five per cent of classes and falls monotonically with the
minimum, so the tail moves as a whole.}
\label{tab:noise}
\end{table}

\section{Runtime}
\label{app:runtime}

We report the cost of a whole configuration rather than of individual operations, because
that is what the run logs measure and it is the number a practitioner cares about.
\Cref{tab:runtime} gives the median wall-clock time per configuration for each phase of the main
experiment. \textsc{Pcc} itself adds no training: the descriptors are read from the classifier's
weight matrix, $g_\theta$ is a closed-form ridge fit over at most $K$ points with $p=9$
features, and the shrinkage grid evaluates eight thresholds. No images, gradients or GPU are
involved at any point after the score dump exists.

The baselines dominate the cost. A configuration that additionally runs the seven baseline
entries takes $476$\,s against $128$\,s without them, so the comparison costs roughly $350$\,s
per split, or about $2.7$ times the whole rest of the pipeline. Most of that is the
kernel-weighted variants, which are evaluated at five bandwidths each. The main sweep is $228$ configurations and
$5.6$ hours of CPU time; the additions of \cref{app:split,app:rowbudget,app:descriptors} add
a further $73$ and about $3$ hours.

\begin{table}[ht]
\centering
\small
\setlength{\tabcolsep}{4pt}
\begin{tabular}{llrr}
\toprule
Phase & Content & Configs & Median \\
\midrule
C & baseline comparison      & $3$  & $481$\,s \\
D & ImageNet-C               & $90$ & $21$\,s \\
E & additional backbones     & $30$ & $21$\,s \\
F & evaluation-depth sweep   & $25$ & $235$\,s \\
G & long-tailed datasets     & $20$ & $63$\,s \\
H & held-out fraction sweep  & $25$ & $131$\,s \\
S & score and $\alpha$ sweep & $35$ & $127$\,s \\
\midrule
\multicolumn{3}{l}{Total, $228$ configurations} & $5.6$\,h \\
\multicolumn{3}{l}{Later additions, $184$ configurations} & $9.5$\,h \\
\bottomrule
\end{tabular}
\caption{Wall-clock cost per configuration, one CPU core, measured from the run reports. Phase D
and E are fast because their evaluation slices are small; phase F is slow because each of its
five depth caps is run at ten splits.}
\label{tab:runtime}
\end{table}

\section{Matching set size without the test set}
\label{app:sigmacalib}

\hyperref[eq:sigma]{Eq.~4} solves for the scalar that brings average set size to the target, and that
average is taken over the evaluation rows. No labels enter --- set size does not need them ---
but the evaluation \emph{inputs} do, which makes the comparison transductive. A deployed system
sees one test point at a time and has no test-set average to match against, so it could not
reproduce the matching, only the thresholds that precede it.

We therefore re-ran the primary configuration with both the target and the shift solved on the
calibration slice, columns restricted to the table's classes and every row kept. Restricting
rows by label would leave nothing for the held-out table, since a class with no calibration rows
is what that table is about; asking instead how many of those classes each calibration row
admits is a quantity a deployment can compute. Coverage is still read off the evaluation rows.

The improvement does not depend on the transductive part. Over the same five splits it is
$+0.058$ $[0.006,0.110]$ either way, agreeing to four decimals, while the shift applied to
\textsc{Pcc} moves from $-0.0019$ to $-0.0025$ and the set size it achieves on the evaluation
rows moves from $1.223$ to $1.215$ --- so the change is real and the statistic is simply not
sensitive to it, the same coarseness \cref{app:descriptors} describes. The lower tail does
register it, at $0.7318$ against $0.7308$.

One caveat belongs here rather than in a footnote. The oracle's shrinkage is selected on the
evaluation slice and this variant does not re-derive that selection, so under calibration
matching the oracle optimises one objective and is read under another; its ceiling collapses
and we do not report it for this arm. That does not affect what the arm is for: \textsc{Pcc} and
the marginal threshold are matched the same way as each other, which is what the comparison
requires. The oracle is a scale, not a comparator.

\section{Every choice the method makes}
\label{app:config}

\Cref{tab:config} lists every quantity a reader would have to guess otherwise. Two of them are
worth a sentence. The shrinkage grid contains $0$, which is what lets the method decline to act
rather than acting badly, and the entry $1$ is included so that the unshrunk correction is
always among the candidates and its cost is visible in \hyperref[tab:ablation]{Tab.~2} rather than assumed.
And the target set size $T$ is not a free parameter: it is the average set size the marginal
threshold itself produces on the labelled calibration rows, so ``matched size'' means matched to
the uncorrected marginal threshold, and nothing is tuned by choosing it.

Nothing in this table was selected on evaluation rows or on held-out classes. The two quantities
that are selected at all, $\hat\lambda$ and $n_{\mathrm{cal}}$, are selected on the labelled
label space, which is why the labelled and held-out tables are reported apart
(\hyperref[sec:splits]{Sec.~3.4}).

\begin{table}[ht]
\centering
\small
\setlength{\tabcolsep}{3pt}
\begin{tabular}{@{}ll@{}}
\toprule
Quantity & Value \\
\midrule
Shrinkage grid $\Lambda$
  & $0,\ .05,\ .1,\ .2,\ .3,\ .5,\ .7,\ 1$ \\
Ridge penalty & $1.0$, fixed, never tuned \\
Descriptor scaling & standardised on fit classes \\
$n_{\mathrm{cal}}$ & $25$, pre-registered \\
Target size $T$ & avg.\ size of $\hat q$ \\
$\alpha$ & $0.10$ unless stated \\
Held-out fraction & $0.30$ unless swept \\
Calib.\ fraction & $0.70$ of row budget \\
Row budget & $250{,}000$, per-class random \\
Splits & seeds $0$--$4$ ($0$--$9$ ablations) \\
Neighbour counts $k$ & $5,10,50$, never tuned \\
$k=1$ & held out as the distance reference \\
$g_\theta$ head & ridge (\cref{app:heads}) \\
Recal.\ slice & $0$ reuse, $0.25$ split \\
\bottomrule
\end{tabular}
\caption{Every hyperparameter and protocol constant. Values are the code's defaults, which are
what produced every number in the paper; the sweeps state their own departures. The ridge
penalty is fixed rather than cross-validated, because the fit has at most $K$ points and nine
features and we would rather report a fixed choice than a tuned one.}
\label{tab:config}
\end{table}

\section{Reproducibility}
\label{app:repro}

All $412$ configurations write one JSON report each, containing the resolved configuration,
the seed, every metric in \cref{tab:metrics} for both tables, the selected $\lambda$ and its
whole curve, the measurability regime, wall-clock time, the interpreter and package versions,
and the git commit of the code that produced it. The pre-registration documents summarised in
\cref{app:prereg} are in the same repository and are dated before the runs.

Two limitations of the comparison should be stated rather than implied. First, baselines are
invoked from their authors' released implementations, which we consider the only defensible way
to report a baseline's behaviour at $n_y=0$, but we have \emph{not} reproduced their published
numbers on their own benchmarks; our claim is about what their code does on our data, not about
whether we have implemented their method as they would. Second, \textsc{Tacp} \cite{liu2026tail}
is discussed in \hyperref[sec:related]{Sec.~2} and absent from \hyperref[tab:competitors]{Tab.~1} because its authors
release no code and it is evaluated on benchmarks we do not have, so any row we produced would
be our own reimplementation held to a lower standard than the rest of the table.

\section{Tail statistics beyond the minimum}
\label{app:tail}

Worst-class coverage is one order statistic and inherits the variance of one class, so a reader
may reasonably ask what a less extreme summary of the lower tail shows. \Cref{tab:tail} reports
five of them on the $300$ held-out ImageNet classes. The arms are the ones \hyperref[tab:competitors]{Tab.~1}
compares and are read at the same matched average set size; the coverage vectors are the ones
that table is built from, saved from the same run so that the two cannot disagree about the
split or the fit.

\begin{table}[ht]
\centering
\small
\setlength{\tabcolsep}{4pt}
\begin{tabular}{@{}lcccc@{}}
\toprule
Statistic & Marginal & \textsc{Pcc} & Oracle & $\Delta$ \textsc{Pcc} \\
\midrule
Worst class        & $0.6579$ & $0.6912$ & $0.7922$ & $+0.0333$ \\
$5$th percentile   & $0.7646$ & $0.7662$ & $0.8421$ & $+0.0016$ \\
$10$th percentile  & $0.7992$ & $0.8155$ & $0.8571$ & $+0.0163$ \\
CVaR$_{5\%}$       & $0.7313$ & $0.7381$ & $0.8230$ & $+0.0068$ \\
CVaR$_{10\%}$      & $0.7565$ & $0.7630$ & $0.8387$ & $+0.0066$ \\
\midrule
Mean over classes  & $0.9039$ & $0.8998$ & $0.8812$ & $-0.0041$ \\
\bottomrule
\end{tabular}
\caption{Lower-tail statistics of per-class coverage on the held-out classes, seed $0$ of the
primary configuration under the reuse variant, all arms at matched average set size.
CVaR$_{q}$ is the mean coverage of the worst $q$ of classes. This is a single split and
therefore carries no interval: it is a diagnostic, and \hyperref[tab:competitors]{Tab.~1} is the
five-split statement.}
\label{tab:tail}
\end{table}

Two things are visible. The tail moves in the same direction as the minimum at every depth,
which is also what \cref{tab:noise} shows as the head rows are perturbed. And the gain is
concentrated near the bottom: \textsc{Pcc} moves the minimum by $+0.0333$ and the mean of the
worst five per cent by $+0.0068$, a fifth as much, where the oracle moves them by $+0.1343$ and
$+0.0917$, two thirds as much. The single quantiles are jumpy at this sample size --- the fifth
percentile of $300$ classes is the fifteenth class, and it moves $+0.0016$ while the tenth
percentile moves $+0.0163$ --- which is why the CVaR rows are the ones to read. The last row is
the price: at matched size, coverage moved into the tail comes out of the mean.

The concentration is what the selection rule asks for. $\hat\lambda$ is chosen to maximise
worst-class coverage on the labelled classes (\cref{app:prereg}), so the method spends its one
scalar on the minimum, while the oracle reads every class's test labels and has no such
bottleneck. Selecting $\hat\lambda$ against a risk-averse surrogate instead --- CVaR$_{5\%}$,
say --- would plausibly trade some of the headline for a flatter tail and less variance across
splits, and is the version we would build next. We did not substitute it here because the
criterion was fixed before the runs (\cref{app:prereg}), and a pre-registered statistic is
worth more than the best of several chosen afterwards.

\section{Both sides of the sup-norm bound}
\label{app:lipschitz}

\hyperref[prop:sup]{Prop.~1} assumes each $F_y$ is $L$-Lipschitz and concludes that two matched offset
vectors give worst-class coverages within $L$ times their sup-norm distance. Both quantities
are observable, so neither needs to be assumed. This section measures both, on the primary
configuration over five splits, and the result is the one \hyperref[sec:correction]{Sec.~3.3} reports: the bound holds
and does not explain where $\hat\lambda$ lands.

Two things are read at each $\lambda$. The sup-norm distance is taken on the \emph{matched}
thresholds, which is what the proposition is stated over; taking it on the raw offsets instead
drops $\sigma(\lambda\hat\delta)-\sigma(\delta^\star)$, a constant added to every class and
therefore a direct shift of a sup norm. And because the proof invokes Lipschitzness only at the
two thresholds being compared, the constant it needs is not a global one: $L_{\mathrm{eff}}$ is
the smallest constant for which the proof's own inequality holds at those points, and
$L_{\mathrm{med}}$ is the median of the same ratio over classes.

\begin{table}[ht]
\centering
\small
\setlength{\tabcolsep}{3pt}
\begin{tabular}{@{}lccccc@{}}
\toprule
$\lambda$ & $\|\bar\delta-\bar\delta^\star\|_\infty$ & $\Delta\sigma$ &
$L_{\mathrm{eff}}$ & $L_{\mathrm{med}}$ & worst-class \\
\midrule
$0.00$ & $0.9421$ & $+0.008$ & $19.1$ & $0.54$ & $0.5825$ \\
$0.05$ & $0.9234$ & $+0.011$ & $10.3$ & $0.49$ & $0.6197$ \\
$0.10$ & $0.9006$ & $+0.009$ & $16.0$ & $0.51$ & $\mathbf{0.6405}$ \\
$0.20$ & $0.8416$ & $-0.014$ & $16.8$ & $0.47$ & $0.6136$ \\
$0.30$ & $0.7776$ & $-0.042$ & $34.7$ & $0.47$ & $0.5700$ \\
$0.50$ & $0.6627$ & $-0.099$ & $16.2$ & $0.43$ & $0.4855$ \\
$0.70$ & $\mathbf{0.5510}$ & $-0.156$ & $13.3$ & $0.42$ & $0.4090$ \\
$1.00$ & $0.6467$ & $-0.242$ & $14.1$ & $0.39$ & $0.1256$ \\
\bottomrule
\end{tabular}
\caption{Both sides of \hyperref[eq:supbound]{Eq.~7} on the held-out label space, five splits at the
primary configuration. $\Delta\sigma$ is the differential size-matching shift. The
$\lambda=0$ and $\lambda=0.1$ rows reproduce \hyperref[tab:competitors]{Tab.~1}: their difference is
$+0.0580$. The labelled label space, where $\hat\lambda$ is selected, behaves the same way:
the distance is smallest at $\lambda=0.7$ and worst-class coverage peaks at $0.1$.}
\label{tab:lipschitz}
\end{table}

The bound holds at every $\lambda$ on both label spaces, which is the check that the two sides
are being read at the same thresholds. It is also useless. Instantiated with
$L_{\mathrm{eff}}$ it exceeds the gap it bounds by a factor of $99$ on the held-out space and
$49$ on the labelled one. The reason is visible in the two columns: $L_{\mathrm{med}}$ is
$0.4$ to $0.5$ and barely moves, so the typical class is well behaved, while
$L_{\mathrm{eff}}$ is ten to thirty-five times larger. With $76$ evaluation rows per class the
empirical $\hat F_y$ is a step function of height $1/76=0.013$, so any class whose two
thresholds fall within one row of each other returns an arbitrarily large ratio. What
$L_{\mathrm{eff}}$ measures at this depth is the discreteness of the estimate, not the
smoothness of $F_y$; $L_{\mathrm{med}}$ is the number to read, and it says a threshold change
of $\varepsilon$ moves a typical class's coverage by about $\varepsilon/2$.

The substantive finding is that the sup norm does not select $\hat\lambda$. It falls
monotonically from $\lambda=0$ to $\lambda=0.7$ and rises again at $1$, so it does have an
interior minimum --- but at $\lambda=0.7$, while worst-class coverage peaks at $0.1$ and has
already fallen by $0.23$ by the time the distance reaches its minimum. Between the two optima
the columns move in opposite directions at every step of the grid. A reader given only
\hyperref[eq:supbound]{Eq.~7} would choose $\lambda=0.7$ and lose most of the effect.

What does move with $\lambda$ is $\Delta\sigma$, from $+0.008$ to $-0.242$: a wider correction
inflates average set size, and \hyperref[eq:sigma]{Eq.~4} pays for it by shifting every threshold down,
the worst class included. That term is four times the size of the entire effect this paper
reports and is exactly what a sup-norm distance between two \emph{already matched} vectors
cannot see, since it is common to all classes. This is why $\lambda=1$ is not merely
suboptimal but catastrophic --- worst-class coverage $0.1256$ against $0.5825$ at
$\lambda=0$ --- even though its sup-norm distance is smaller than at $\lambda=0$.

\Cref{tab:lipschitz} also answers what the shrinkage grid itself is worth. Its range brackets
the optimum on both sides: $\lambda=0$ is the null and $\lambda=1$ is catastrophic, so there is
nothing to gain by extending it. Its resolution is load-bearing, and unevenly so. Worst-class
coverage peaks at $\lambda=0.1$ on \emph{both} label spaces --- the labelled one, where
$\hat\lambda$ is selected, and the held-out one, where the headline is read --- so the selection
rule transfers rather than merely being self-consistent, and $\hat\lambda=0.10$ on all five
splits. But the neighbouring grid points cost $0.021$ and $0.027$, which is a third of the
effect being reported, and $\lambda=0.5$ already scores below the null on both spaces. A grid
whose smallest non-zero value is $0.5$ therefore selects $\lambda=0$ and forgoes the effect
entirely. The density near zero is not a formality; a uniformly spaced grid of the same size
would have found nothing.

\hyperref[prop:sup]{Prop.~1} is therefore sound and its bound verified, but it is a qualitative statement
about why a worst-class objective is governed by a maximum rather than a mean, and it should
not be read as an account of where $\hat\lambda$ lands. The selection rule of
\hyperref[eq:lambda]{Eq.~6} is empirical for this reason, and this measurement is the argument for
keeping it so.

\end{document}